\documentclass[twoside,11pt]{article}
\usepackage{jmlr2e2}
\providecommand{\SetAlgoVlined}{\SetLine}

\usepackage[latin1]{inputenc}
\usepackage[T1]{fontenc} 
\usepackage{lmodern}
\usepackage{natbib}

\usepackage{amscd}

\title{Off-policy Learning with Eligibility Traces:\\ A Survey}

\def\qed{\endproof}

\jmlrheading{N}{20??}{?-?}{?/?}{?/?}{Matthieu Geist and Bruno Scherrer}
\ShortHeadings{Off-policy learning with traces}{M. Geist and B. Scherrer} \firstpageno{1} \editor{?}

\usepackage{mathtools}
\mathtoolsset{showonlyrefs=true}

\usepackage{graphicx} 

\usepackage{amssymb}

\usepackage[ruled,algo2e]{algorithm2e}

\newcommand{\argmin}{\operatorname*{argmin}}

\newcommand{\diag}{\operatorname*{diag}}
\newcommand{\commentout}[1]{}

\usepackage[usenames]{color} 
\definecolor{Blue}{rgb}{0,0.16,0.90}
\definecolor{Red}{rgb}{0.90,0.16,0}
\newcommand{\blue}[1]{{\color{Blue} #1}}
\newcommand{\red}[1]{{\color{Red} #1}}

\newcommand{\eqrefb}[1]{Equation~\eqref{eq:#1}}
\newcommand{\refsec}[1]{Section~\ref{sec:#1}}
\newcommand{\refsubsec}[1]{Section~\ref{subsec:#1}}

\newcommand{\tphi}{\tilde{\Phi}}

\def\ie{\textit{i.e.}}
\def\eg{for instance}
\def\cf{\textit{cf.} }
\def\aka{also known as }

\author{\name{Matthieu Geist} \email  matthieu.geist@supelec.fr\\
\addr IMS-MaLIS Research Group, Supélec (France) 
\AND
\name{Bruno Scherrer} \email bruno.scherrer@inria.fr \\
 \addr MAIA project-team, INRIA Lorraine (France)
}

\begin{document}

\maketitle

\begin{abstract}

In the framework of Markov Decision Processes, \emph{off-policy}
learning, that is the problem of learning a linear approximation of
the value function of some fixed policy from one trajectory possibly
generated by some other policy. We briefly review \emph{on-policy}
learning algorithms of the literature (gradient-based and
least-squares-based), adopting a unified algorithmic view.  Then, we
highlight a systematic approach for adapting them to \emph{off-policy}
learning \emph{with eligibility traces}. This leads to some known
algorithms -- off-policy LSTD($\lambda$), LSPE($\lambda$),
TD($\lambda$), TDC/GQ($\lambda$) -- and suggests new extensions --
off-policy FPKF($\lambda$), BRM($\lambda$), gBRM($\lambda$),
GTD2($\lambda$). We describe a comprehensive algorithmic derivation of
all algorithms in a recursive and memory-efficent form, discuss their
known convergence properties and illustrate their relative empirical
behavior on Garnet problems. Our experiments suggest that the most
standard algorithms on and off-policy LSTD($\lambda$)/LSPE($\lambda$)
-- and TD($\lambda$) if the feature space dimension is too large for a
least-squares approach -- perform the best.

\end{abstract}


\section{Introduction}
\label{sec:intro}

We consider the problem of learning a linear approximation of
the value function of some fixed policy in a Markov Decision Process
(MDP) framework, in the most general situation where
learning must be done from a single trajectory possibly generated by
some other policy, \aka \emph{off-policy} learning.
Given samples, well-known methods for estimating a value
function are temporal difference (TD) learning and Monte
Carlo~\citep{Sutton:1998}. TD learning with eligibility
traces~\citep{Sutton:1998}, known as TD($\lambda$), constitutes a nice
bridge between both approaches, and by controlling the bias/variance
trade-off \citep{kearns:2000}, their use can significantly speed up
learning.  When the value function is approximated through a linear
architecture, the depth $\lambda$ of the eligibility traces is also
known to control the quality of approximation~\citep{vanroy:1997}.
Overall, the use of these traces 
often
plays an important practical role.

\begin{table}[tbh]
  \begin{center}
    \begin{tabular}{c||c|c}
      ~ & gradient-based & least-squares-based
      \\ \hline \hline
      bootstrapping & TD & FPKF
      \\
      ~ & \citep{Sutton:1998} & \citep{Choi:2006}
      \\ \hline
      residual & gBRM & BRM
      \\
      ~ & \citep{Baird:1995} & \citep{Engel:2005,geist:jair}
      \\ \hline
      projected fixed point & TDC/GTD2 & LSTD \citep{Bradtke:1996}
      \\
      ~ & \citep{Sutton:2009a} & LSPE \citep{Nedic:2003}
    \end{tabular}
    \caption{Taxonomy of linearly parameterized estimators for value function approximation~\citep{geist:vfa}.}
    \label{tab:taxonomy_simple}
  \end{center}
\end{table}

There has been a significant amount of research on parametric linear
approximation of the value function, \emph{without eligibility
traces} (in the on- or off-policy case). We follow the taxonomy
proposed by~\citet{geist:vfa}, briefly recalled in
Table~\ref{tab:taxonomy_simple} and further developped in
\refsec{methods}. Value function approximators can be
categorized depending on the cost function they minimize (based on
bootstrapping, on a Bellman residual minimization or on a projected
fixed point approach) and on how it is minimized
(gradient-descent-based or linear least-squares).
Most of these algorithms have been extended to take into account
eligibility traces, in the on-policy case. Works on extending these
approaches (based on eligibility traces) to off-policy learning are
scarser. They are summarized in Table~\ref{tab:sumary} (algorithms in
black).
The first motivation of this article is to argue that it is
conceptually simple to extend \emph{all} the algorithms of
Table~\ref{tab:taxonomy_simple} so that they can be applied to the
off-policy setting \emph{and} use eligibility traces. If this allows
rederiving existing algorithms (in black in Table~\ref{tab:sumary}),
it also leads to new candidate algorithms (in red in
Table~\ref{tab:sumary}). The second
motivation of this work is to discuss the subtle differences between
these intimately-related algorithms, and to
provide some comparative insights on their empirical behavior (a
topic that has to our knowledge not been considered in the
literature, even in the simplest \emph{on-policy} and
\emph{no-trace} situation).

\begin{table}[tbh]
  \begin{center}
    \begin{tabular}{c||c|c}
      ~ & gradient-based & least-squares-based
      \\ \hline \hline
      bootstrapping & off-policy TD($\lambda$) & \red{off-policy
      FPKF($\lambda$)}
      \\
      ~ & \citep{bertsekas:09proj} & ~
      \\ \hline
      residual & \red{off-policy gBRM($\lambda$)} & \red{off-policy
      BRM($\lambda$)}
      \\ \hline
      projected fixed point & GQ($\lambda$) (a.k.a. off-policy TDC($\lambda$)) & off-policy LSTD($\lambda$)
      \\
      ~ & \citep{Maei:2010} & off-policy LSPE($\lambda$)
      \\
      ~ & \red{off-policy GTD2($\lambda$)} & \citep{Yu:2010}
    \end{tabular}
    \caption{Surveyed off-policy and eligibility-traces-based approaches. Algorithms in black have been published before (provided references),
    algorithms in red are new.}
    \label{tab:sumary}
  \end{center}
\end{table}

The rest of this article is organized as follows.  \refsec{methods}
introduces the background of Markov Decision Processes, describes the
state-of-the-art algorithms for learning without eligibility traces,
and gives the fundamental idea to extend the methods to the off-policy
situation with eligibility traces. \refsec{algo} details this
extension for the least-squares based approaches: the resulting
algorithms are formalized, and we derive recursive and memory-efficient formula for their implementation  (this allows online learning without loss of
generality, all the more that half of these algorithms are recursive
by their very definition), and we discuss their convergence
properties.
\refsec{algo_grad} does the same job for
stochastic gradient based approaches, which offers a smaller
computational cost (linear per update, instead of quadratic).
Last but not least, \refsec{exp} describes an empirical comparison and \refsec{conclusion} concludes. 

\section{Background}
\label{sec:methods}

We consider a Markov Decision Process (MDP), that is a tuple $\{S,A,P,R,\gamma\}$ in
which $S$ is a finite state space identified with $\{1,2,\dots,N\}$,  $A$ a finite action space, $P\in \mathcal{P}(S)^{S\times A}$ the set of transition
probabilities, $R\in \mathbb{R}^{S\times A}$ the
reward function and $\gamma$ the discount factor. A mapping
$\pi\in\mathcal{P}(A)^S$ is called a policy. For any policy $\pi$, let $P^\pi$ be the corresponding stochastic transition matrix, and $R^\pi$ the vector of mean reward when following $\pi$, \ie, of components $E_{a|\pi,s}[R(s,a)]$.
The value $V^\pi(s)$ of state $s$ for a policy $\pi$ is the expected discounted
cumulative reward starting in state $s$ and then following the
policy $\pi$:
$$
V^\pi(s) = E_\pi\left[\sum_{i=0}^\infty \gamma^i r_i|s_0=s\right],
$$
where $E_\pi$ denotes the expectation induced by policy $\pi$.
The value function satisfies the (linear) Bellman equation:
$$
\forall s,~~V^\pi(s) = E_{s',a|s,\pi}[R(s,a) + \gamma V^\pi(s')].
$$
It can be
rewritten as the fixed-point of the Bellman evaluation operator:
$V^\pi = T^\pi V^\pi$ where for all $V,~~T^\pi V=R^\pi + \gamma P^\pi V$.

In this article, we are interested in learning an approximation of this value function
$V^\pi$ under some constraints. First, we assume our
approximation to be linearly parameterized:
$$
\forall s,~~\hat{V}_\theta(s) =\theta^T \phi(s)
$$
with $\theta\in\mathbb{R}^p$ being the parameter vector and
$\phi(s) \in\mathbb{R}^p $ the feature vector in state $s$. Also, we want to estimate the value
function $V^\pi$ (or equivalently the associated parameter $\theta$)
from a single finite trajectory generated using a possibly different
behavioral policy $\pi_0$. Let $\mu_{0}$ be the stationary
distribution of the stochastic matrix $P_0=P^{\pi_0}$ of the
\emph{behavior policy} $\pi_0$ (we assume it exists and is unique).
Let $D_0$ be the diagonal matrix of which the elements are
$(\mu_0(s_i))_{1 \le i \le N}$. Let $\Phi$ be the matrix of feature
vectors:
$$
\Phi = [\phi(1)\dots\phi(N)]^T.
$$
The projection
$\Pi_0$ onto the hypothesis space spanned by $\Phi$ with respect to the
$\mu_0$-quadratic norm, which will be central for the understanding of the algorithms, has the following closed-form:
\begin{equation}
  \Pi_0 = \Phi(\Phi^T D_0 \Phi)^{-1} \Phi^T D_0.
\end{equation}
If $\pi_0$ is different from $\pi$, it is called an off-policy setting.
Notice that all algorithms considered in this paper use this $\Pi_0$
projection operator, that is the projection according to the
observed data\footnote{It would certainly be interesting to consider the
projection according to the stationary distribution of $\pi$, the
(unobserved) target policy: this would reduce off-policy learning to
on-policy learning. However, this would require reweighting samples
according to the stationary distribution of the target policy $\pi$,
which is unknown and probably as difficult to estimate as the value
function itself. As far as we know, the only work to move in this
direction is the off-policy approach of~\citet{Kolter:2011}: samples
are weighted such that the projection operator composed with the
Bellman operator is non-expansive (so, weaker than finding the
projection of the stationary distribution, but offering some
guarantees). In this article, we consider only the
$\Pi_0$ projection.}.


\paragraph{Standard Algorithms for on-policy Learning without Traces.}
We now review existing on-policy linearly
parameterized temporal difference learning algorithms (see
Table~\ref{tab:taxonomy_simple}). In this case, the behavior and
target policies are the same, so we omit the subscript $0$ for the
policy ($\pi$) and the projection ($\Pi$). We assume that a
trajectory $(s_1,a_1,$ $r_1,s_2,\dots,$ $s_i,a_i,r_i,$
$s_{i+1},\dots, s_n,a_n,r_n,s_{n+1})$ sampled according to the policy $\pi$ is
available, and will explain how to compute the $i^{th}$ iterate for several algorithms. For all $j \le i$, let us introduce the empirical Bellman operator at step $j$:
\begin{align}
  \hat{T}_j~:~ \mathbb{R}^S &\rightarrow  \mathbb{R} \\
  V &\mapsto r_j + \gamma
  V(s_{j+1}) 
\end{align}
so that $\hat{T}_j V$ is an unbiased estimate of $T V(s_j)$.

{\bf Projected fixed point approaches} aim at
%
finding the fixed-point of the operator being the composition of the
projection onto the hypothesis space and of the Bellman operator. In
other words, they search for the fixed-point $\hat{V}_\theta = \Pi T
\hat{V}_\theta$, $\Pi$ being the just introduced projection
operator. Solving the
following fixed-point problem:
\begin{equation}
  \theta_i = \argmin_{\omega\in\mathbb{R}^p} \sum_{j=1}^i
  \left(\hat{T}_j \hat{V}_{\blue{\theta_{i}}} - \hat{V}_\omega(s_j)\right)^2
  \label{eq:lstd_0}
\end{equation}
with a least-squares approach, corresponds to 
the Least-Squares Temporal Differences (LSTD) algorithm
of~\citet{Bradtke:1996}. Recently, \citet{Sutton:2009a}
proposed two algorithms reaching the same objective, Temporal Difference with
gradient Correction (TDC) and Gradient Temporal Difference 2 (GTD2), by performing  a stochastic gradient descent of the function $\theta \mapsto \|\hat{V}_\theta-\Pi T\hat{V}_\theta \|^2$ which is minimal (and equal to 0) when $\hat{V}_\theta = \Pi T
\hat{V}_\theta$.

A related approach consists in building a recursive algorithm that repeatedly mimicks the iteration $\hat{V}_{\theta_i} \simeq \Pi T \hat{V}_{\theta_{i-1}}$.
In practice,  we aim at minimizing
\begin{equation}
\omega \mapsto \sum_{j=1}^i
  \left(\hat{T}_j \hat{V}_{\blue{\theta_{i-1}}} - \hat{V}_\omega(s_j)\right)^2
  \label{eq:lspe_0}.
\end{equation}
Performing the minimization exactly through a least-squares method
leads to the Least-Squares Policy Evaluation (LSPE) algorithm
of~\citet{ioffe}.
If this minimization is approximated by a stochastic gradient descent, this leads to the classical Temporal Difference (TD) algorithm~\citep{Sutton:1998}.

%


{\bf Bootstrapping approaches} consist in treating value function
approximation after seeing the $i^{th}$ transition as a supervised learning problem, by replacing the unobserved values $V^\pi(s_j)$ at states $s_j$ by some estimate computed from the trajectory until the transition $(s_j,s_{j+1})$, the best such estimate being $\hat{T}_j\hat{V}_{\theta_{j-1}}$. This
amounts to minimizing the following function:
\begin{equation}
\omega \mapsto \sum_{j=1}^i
  \left(\hat{T}_j \hat{V}_{\blue{\theta_{j-1}}} - \hat{V}_\omega(s_j)\right)^2
  \label{eq:fpkf_0}.
\end{equation}
\citet{Choi:2006} proposed
the Fixed-Point Kalman Filter (FPKF), a least-square variation of TD
that minmizes exactly the function of~\eqrefb{fpkf_0}.  If the
minimization is approximated by a stochastic gradient descent, this
gives -- again -- the classical TD algorithm~\citep{Sutton:1998}.


%
Finally, {\bf residual approaches} aim at minimizing the distance between
the value function and its image through the Bellman operator, $\|V
- T V\|_{\mu_0}^2$. Based on a trajectory, this suggests the following function to minimize
\begin{equation}
\omega \mapsto  \sum_{j=1}^i
  \left(\hat{T}_j \hat{V}_{\blue \omega} - \hat{V}_\omega(s_j)\right)^2
  \label{eq:brm_0},
\end{equation}
which is a biased surrogate of the objective  $\|V
- T V\|_{\mu_0}^2$ (\eg, see \cite{antos2006}).
This cost function has originally been proposed
by~\citet{Baird:1995} who minimized it using a stochastic gradient
approach (this algorithm being referred here as gBRM for
gradient-based Bellman Residual Minimization). 
Both the
\emph{parametric} Gaussian Process Temporal Differences (GPTD)
algorithm of~\citet{Engel:2005} and the \emph{linear} Kalman
Temporal Differences (KTD) algorithm of~\citet{geist:jair} can be
shown to minimize the above cost using a least-squares approach, and are thus the very same algorithm\footnote{This is only true in the linear case. GPTD and KTD were both introduced in a more general setting: GPTD is nonparametric and KTD is motivated by the goal of handling nonlinearities.}, that we will refer
to as BRM (for Bellman Residual Minimization) in the remaining of
this paper.

To sum up, it thus appears that after the $i^{th}$ transition has been observed, the above mentioned algorithms behave according to the following pattern:
\begin{align}
\mbox{move from $\theta_{i-1}$ to $\theta_i$ towards the minimum of } \omega \mapsto \sum_{j=1}^i
  \left(\hat{T}_j \hat{V}_{\red \xi} - \hat{V}_{\omega}(s_j)\right)^2,
  \label{eq:unify1}
\end{align}
either through a least-squares approach or a stochastic gradient
descent.
Each of the algorithms mentionned above is obtained by substituting $\blue{\theta_i}\text{, } \blue{\theta_{i-1}}\text{, } \blue{\theta_{j-1}}\text{ or }\blue{\omega}$ for  $\red \xi$.


\paragraph{Towards Off-policy Learning with Traces}

It is now easy to preview, at least at a high level, how one may extend the previously described algorithms so that they can deal with
eligibility traces and off-policy learning.
The idea of eligibility traces amounts to looking for the fixed-point of the
following variation of the Bellman operator~\citep{Bertsekas:1996}
\begin{equation}
\forall V \in \mathbb{R}^S,~~  T^\lambda V = (1-\lambda)
  \sum_{k=0}^\infty \lambda^k T^{k+1} V
  \label{eq:tl_bert}
\end{equation}
that makes a geometric average with parameter $\lambda \in (0,1)$ of the
powers of the original Bellman operator $T$. Clearly, any
fixed-point of $T$ is a fixed-point of $T^\lambda$ and vice-versa.
After some simple algebra, one can see that:
\begin{align}
T^\lambda V & =  (I-\lambda \gamma P)^{-1}(R+(1-\lambda)\gamma P V) \label{eq:tlambda} \\
 &=  V+(I-\lambda \gamma P)^{-1}(R+\gamma P V - V).
\end{align}
This leads to the following well-known \emph{temporal
difference}-based expression in some state $s$
\begin{align}
T^\lambda V(s) &= V(s)  + E_\pi\left[\sum_{k=i}^\infty (\gamma \lambda)^{k-i} \Big(r_k + \gamma V(s_{k+1}) - V(s_k)\Big) \Big|s_i=s\right]\\
& = V(s) + \sum_{k=i}^\infty (\gamma  \lambda)^{k-i} \delta_{ik}(s)
\end{align}
where we recall that $E_\pi$ means that the expectation is done according to the target policy $\pi$, and where we $\delta_{ik}(s)=E_\pi\left[r_k + \gamma V(s_{k+1}) - V(s_k) \Big|s_i=s\right]$ is the expected temporal-difference \citep{Sutton:1998}. With $\lambda=0$, we recover the
Bellman evaluation equation. With $\lambda = 1$, this
is the definition of the value function as the expected and
discounted cumulative reward: $T^1 V(s) = E_\pi[\sum_{k=i}^\infty
\gamma ^{k-i} r_k | s_i=s]$.

  %
As before, we assume that we are given a trajectory  $(s_1,a_1,r_1,s_2,\dots,$ $s_j,a_j,r_j,s_{j+1} \dots,$ $ s_n,a_n,r_n, s_{n+1})$, except now that it may be generated from some behaviour policy possibly different from the target policy $\pi$ of which we want to estimate the value. We are going to describe how to compute the $i^{th}$ iterate for several algorithms. For any $i \le k$, unbiased estimates
of the temporal difference terms $\delta_{ik}(s_k)$ can be computed through
importance sampling~\citep{Ripley:87}. Indeed, for all $s,a$, let us introduce
the following weight:
\begin{equation}
\rho(s,a)= \frac{\pi(a|s)}{\pi_0(a|s)}.
\end{equation}
In our trajectory context, for any $j$ and $k$, write
\begin{equation}
  \rho_j^k = \prod_{l=j}^k \rho_l
  \text{ with }
  \rho_l = \rho(s_l,a_l)
\end{equation}
with the convention that if $k<j$, $\rho_j^k=1$. With these notations,
\begin{align}
\hat \delta_{ik} =\rho_i^k \hat{T}_k V - \rho_i^{k-1}V(s_k)
\end{align}
is an unbiased
estimate of $\delta_{ik}(s_k)$, from which we may build an estimate
$\hat T^\lambda_{j,i} V$ of $T^\lambda V(s_j)$ (we will describe this very
construction separately for the least-squares and the stochastic
gradient as they slightly differ). Then, by replacing the empirical
operator $\hat{T}_j$ in~\eqrefb{unify1} by $\hat{T}_{j,i}^{\lambda}$, we
get the general pattern for off-policy trace-based algorithms:
\begin{align}
\mbox{move from $\theta_{i-1}$ to $\theta_i$ towards the minimum of }\omega \mapsto \sum_{j=1}^i  \left(\hat{T}^\lambda_{j,i} \hat{V}_{\red \xi} - \hat{V}_{\omega}(s_j)\right)^2,
  \label{eq:cost_function_generic}
\end{align}
either through a least-squares approach or a stochastic gradient
descent after having instantiated $\red \xi
= \blue{\theta_i}\text{, } \blue{\theta_{i-1}}\text{, }
\blue{\theta_{j-1}}\text{ or }\blue{\omega}$.
This process, including in particular the precise definition of the
empirical operator $\hat{T}_{j,i}^{\lambda}$, will be further
developped in the next two sections\footnote{Note that we let the empirical operator $\hat{T}_{j,i}^{\lambda}$ depends on the index $j$ of the sample (as before) but also on the step $i$ of the algorithm. This will be particularly useful for the derivation of the recursive and memory-efficient least-squares based algorithms that we present in the next section.}.
Since they are easier to derive, we begin by focusing on least-squares algorithms (right
column of Table~\ref{tab:sumary}) in \refsec{algo}.
Then, \refsec{algo_grad} focuses on
stochastic gradient descent-based algorithms (left column of
Table~\ref{tab:sumary}).

\section{Least-squares-based extensions to eligibility traces and off-policy learning}
\label{sec:algo}


In this section, we first consider the case of least-squares solving of the pattern described in~\eqrefb{cost_function_generic}. At their  $i^{th}$ step, the algorithms that we are about to describe will compute the parameter $\theta_i$ by \emph{exactly}  solving the following problem:
\begin{align}
\theta_i = \argmin_{\omega \in \mathbb{R}^p}\sum_{j=1}^i  \left(\hat{T}^\lambda_{j,i} \hat{V}_{\red \xi} - \hat{V}_{\omega}(s_j)\right)^2
\end{align}
where we define the following empirical \emph{truncated} approximation of $T_\lambda$:
\begin{align}
\hat{T}^\lambda_{j,i}~:~ \mathbb{R}^S &\rightarrow  \mathbb{R} \\
  V &\mapsto V(s_j) + \sum_{k=j}^i
  (\gamma\lambda)^{k-j}\hat \delta_{jk} = V(s_j) + \sum_{k=j}^i
  (\gamma\lambda)^{k-j}\Big(\rho_j^k \hat{T}_k V - \rho_j^{k-1}V(s_k)\Big).
\end{align}
Though different definitions of this operator may lead to practical implementations, note that $\hat{T}^\lambda_{j,i}$ only uses samples seen before time $i$: this very feature -- considered by all existing works in the literature -- will enable us to derive recursive and low-memory algorithms.

Recall that a linear parameterization is chosen here,
$\hat{V}_\xi(s_i) = \xi^T \phi(s_i) $. We adopt the following
notations:
\begin{equation}
  \phi_i = \phi(s_i)
  \text{, }
  \Delta\phi_i = \phi_i - \gamma \rho_i \phi_{i+1}
  \text{ and }
  \tilde{\rho}_j^{k-1} = (\gamma \lambda)^{k-j}\rho_j^{k-1}
\end{equation}
The generic cost function to be solved is therefore:
\begin{equation}
  \theta_i = \argmin_{\omega\in\mathbb{R}^p} J(\omega;\red \xi) \text{~~~with~~~}  \label{eq:cost_lin_gen}
  J(\omega;\red \xi) = \sum_{j=1}^i
  (\phi_j^T \red \xi + \sum_{k=j}^i \tilde{\rho}_j^{k-1}(\rho_k r_k - \Delta\phi_{k}^T \red \xi)
  - \phi_j^T\omega)^2.
\end{equation}
Before deriving existing and new least-squares-based algorithms, as
announced, some technical lemmata are required.


The first lemma allows computing directly the inverse of a rank-one
perturbated matrix.
\begin{lemma}[Sherman-Morrison]
\label{lemma:sm}
Assume that $A$ is an invertible $n\times n$ matrix and that
$u,v\in\mathbb{R}^n$ are two vectors satisfying $1 + v^T A^{-1}u
\neq 0$. Then:
\begin{equation}
(A + u v^T)^{-1} = A^{-1} - \frac{A^{-1} u v^T A^{-1}}{1+ v^T A^{-1}
u}
\end{equation}
\end{lemma}
The next lemma is simply a rewriting of imbricated sums. However, it
is quite important here as it will allow stepping from the operator
$\hat{T}^{\lambda}_{j,i}$ (operator which depends on future of $s_j$) -- \emph{forward view} of eligibility traces -- to the
 recursion over parameters using eligibility traces
(dependence on only past samples) -- \emph{backward view} of eligibility
traces -- (see \cite[Ch.7]{Sutton:1998}  for further discussion on
\emph{backward/forward views}).
\begin{lemma}\label{lemma:sum}
  Let $f\in\mathbb{R}^{\mathbb{N}\times \mathbb{N}}$ and $n\in\mathbb{N}$. We have:
  \begin{equation}
    \sum_{i=1}^n \sum_{j=i}^n f(i,j) = \sum_{i=1}^n \sum_{j=1}^i
    f(j,i)
  \end{equation}
\end{lemma}

We are now ready to mechanically derive the off-policy algorithms LSTD($\lambda$), LSPE($\lambda$), FPKF($\lambda$) and BRM($\lambda$). This is what we do in the following subsections.

\subsection{{Off-policy LSTD($\lambda$)}}

The off-policy LSTD($\lambda$) algorithm corresponds to instantiating
Problem~\eqref{eq:cost_lin_gen} with $\red \xi = \blue{\theta_i}$:
\begin{equation}
  \theta_i = \argmin_{\omega\in\mathbb{R}^p} \sum_{j=1}^i
  (\phi_j^T \blue{\theta_i} + \sum_{k=j}^i \tilde{\rho}_j^{k-1}(\rho_k r_k - \Delta\phi_{k}^T \blue{\theta_i})
  - \phi_j^T\omega)^2.
\end{equation}
This can be solved by zeroing the gradient respectively to $\omega$:
\begin{align}
  \theta_i &= (\sum_{j=1}^i \phi_j \phi_j^T)^{-1} \sum_{j=1}^i
  \phi_j(\phi_j^T \theta_i + \sum_{k=j}^i \tilde{\rho}_j^{k-1}(\rho_k r_k - \Delta\phi_{k}^T
  \theta_i))
  \\
  \Leftrightarrow~~~~~
  0 &= \sum_{j=1}^i \sum_{k=j}^i \phi_j \tilde{\rho}_j^{k-1}(\rho_k r_k - \Delta\phi_{k}^T \theta_i),
\end{align}
which, through Lemma~\ref{lemma:sum}, is equivalent to:
\begin{equation}
  0 = \sum_{j=1}^i (\sum_{k=1}^j \phi_k \tilde{\rho}_k^{j-1})(\rho_j r_j - \Delta\phi_{j}^T \theta_i).
\end{equation}
Introducing the (importance based) eligibility vector $z_j$:
\begin{equation}
\label{eq:traces}
  z_j = \sum_{k=1}^j \phi_k \tilde{\rho}_k^{j-1}
  = \sum_{k=1}^j \phi_k (\gamma\lambda)^{j-k}
  \prod_{m=k}^{j-1}\rho_m
  = \gamma \lambda \rho_{j-1} z_{j-1} + \phi_j,
\end{equation}
one obtains the following batch estimate:
\begin{equation}
\label{eq:LSTD:theta}
  \theta_i = (\sum_{j=1}^i z_j \Delta\phi_j^T)^{-1} \sum_{j=1}^i z_j \rho_j r_j=(A_i)^{-1}b_i
\end{equation}
where
\begin{equation}
\label{eq:LSTD:defAb}
A_i=\sum_{j=1}^i z_j \Delta\phi_j^T ~~~\mbox{and}~~~b_i=\sum_{j=1}^i z_j \rho_j r_j.
\end{equation}
Thanks to Lemma~\ref{lemma:sm}, the inverse $M_i=(A_i)^{-1}$ can be computed recursively:
\begin{equation}
  M_i = (\sum_{j=1}^i z_j \Delta\phi_j^T)^{-1} = M_{i-1} - \frac{M_{i-1}  z_i \Delta\phi_i^T M_{i-1} }{1+  \Delta\phi_i^T M_{i-1} z_i}.
\end{equation}
This can be used to derive a recursive estimate:
\begin{align}
    \theta_i &= (\sum_{j=1}^i z_j \Delta\phi_j^T)^{-1}
    \sum_{j=1}^i z_j \rho_j r_j
    = (M_{i-1} - \frac{M_{i-1}  z_i \Delta\phi_i^T M_{i-1} }{1+ \Delta\phi_i^T M_{i-1}
    z_i}) ( \sum_{j=1}^{i-1} z_j r_j \rho_j + z_i \rho_i r_i)
    \\
    &= \theta_{i-1} + \frac{M_{i-1}  z_i}{1+ \Delta\phi_i^T M_{i-1}
    z_i}(\rho_i r_i - \Delta\phi_i^T\theta_{i-1}).
\end{align}
Writing $K_i$ the gain $\frac{M_{i-1}  z_i}{1+ \Delta\phi_i^T
M_{i-1} z_i}$, this gives Algorithm~\ref{algo:lstd}.


\begin{algorithm2e}[tbh]
  \SetAlgoVlined
  \caption{Off-policy LSTD($\lambda$)}
  \label{algo:lstd}
  \BlankLine
  {\textbf{Initialization}}\;
  Initialize vector $\theta_0$ and matrix $M_0$ \;
  Set $z_0=0$\;
  \BlankLine
  \For{$i=1,2,\dots$}{
    \BlankLine
    \textbf{{Observe}} $\phi_i, r_i, \phi_{i+1}$ \;
    \BlankLine
    {\textbf{Update traces}} \;
    $
    z_i = \gamma\lambda \rho_{i-1} z_{i-1} + \phi_i
    $ \;
    \BlankLine
    {\textbf{Update parameters}} \;
    $
    K_i = \frac{M_{i-1} z_i}{1+\Delta\phi_i^T M_{i-1} z_i}
    $ \;
    $\theta_i = \theta_{i-1} + K_i (\rho_i r_i - \Delta\phi_i^T \theta_{i-1})$ \;
    $
    M_i = M_{i-1} - K_i(M_{i-1}^T\Delta\phi_i)^T
    $\;
  }
\end{algorithm2e}

This algorithm has been proposed and analyzed recently
by~\citet{Yu:2010}.  The author proves the following result: if the
\emph{behavior} policy $\pi_0$ induces an irreducible Markov chain
and chooses with positive probability any action that may be chosen
by the \emph{target} policy $\pi$, and if the compound (linear)
operator $\Pi_{0} T^\lambda$ has a unique fixed-point\footnote{It is
not always the case, see \cite{vanroy:1997} for a counter-example.}, then off-policy LSTD($\lambda$) converges to it
almost surely. Formally, it converges to the solution $\theta^*$ of
the so-called \emph{projected fixed-point} equation:
\begin{equation}
\label{eq:projfixpoint}
  V_{\theta^*} = \Pi_{0} T^\lambda V_{\theta^*}.
\end{equation}
Using the expression of the projection $\Pi_0$ and the form of the Bellman operator in \eqrefb{tlambda},
it can be seen that $\theta^*$ satisfies (see \cite{Yu:2010} for details)
\begin{equation}
\theta^*  = A^{-1}b
\end{equation}
where
\begin{equation}
\label{eq:LSTD:Ab}
A=\Phi^T D_0 (I-\gamma P)(I-\lambda\gamma P)^{-1}\Phi \mbox{~~~and~~~}b=\Phi^T D_0 (I-\lambda\gamma P)^{-1} R.
\end{equation}
The core of the analysis of \citet{Yu:2010} consists in showing that $\frac{1}{i}A_i$ and $\frac{1}{i}b_i$ defined in \eqrefb{LSTD:defAb} respectively converge to $A$ and $b$ almost surely.
Through
\eqrefb{LSTD:theta}, this implies the convergence of
$\theta_i$ to $\theta^*$.

\subsection{Off-policy LSPE($\lambda$)}

The off-policy LSPE($\lambda$) algorithm  corresponds to the instantiation $\red \xi
= \blue{\theta_{i-1}}$ in Problem~\eqref{eq:cost_lin_gen}:

\begin{equation}
  \theta_i = \argmin_{\omega\in\mathbb{R}^p} \sum_{j=1}^i
  (\phi_j^T \blue{\theta_{i-1}} + \sum_{k=j}^i \tilde{\rho}_j^{k-1}(\rho_k r_k - \Delta\phi_{k}^T \blue{\theta_{i-1}})
  - \phi_j^T\omega)^2.
\end{equation}
This can be solved by zeroing the gradient respectively to $\omega$:
\begin{align}
  \theta_i &= (\sum_{j=1}^i \phi_j \phi_j^T)^{-1} \sum_{j=1}^i
  \phi_j(\phi_j^T \theta_{i-1} + \sum_{k=j}^i \tilde{\rho}_j^{k-1}(\rho_k r_k - \Delta\phi_{k}^T
  \theta_{i-1}))
  \\
  &= \theta_{i-1} + (\sum_{j=1}^i \phi_j \phi_j^T)^{-1}
  \sum_{j=1}^i \sum_{k=j}^i \phi_j \tilde{\rho}_j^{k-1}(\rho_k r_k - \Delta\phi_{k}^T \theta_{i-1}).
\end{align}
Lemma~\ref{lemma:sum} can be used (recall the definition of the
eligibility vector $z_j$ in \eqrefb{traces}):
\begin{align}
  \theta_i &= \theta_{i-1} + (\sum_{j=1}^i \phi_j \phi_j^T)^{-1}
  \sum_{j=1}^i \sum_{k=1}^j \phi_k \tilde{\rho}_k^{j-1}(\rho_j r_j - \Delta\phi_{j}^T \theta_{i-1})
  \\
  &= \theta_{i-1} + (\sum_{j=1}^i \phi_j \phi_j^T)^{-1}
  \sum_{j=1}^i z_j(\rho_j r_j - \Delta\phi_{j}^T \theta_{i-1}).
\end{align}
Define the matrix $N_i$ as follows:
\begin{equation}
\label{eq:LSPE:N}
N_i = (\sum_{j=1}^i \phi_j \phi_j^T)^{-1} = N_{i-1} -
  \frac{N_{i-1} \phi_i \phi_i^T N_{i-1}}{1 + \phi_i^T N_{i-1} \phi_i},
\end{equation}
where the second equality follows from Lemma~\ref{lemma:sm}.
Let $A_i$ and $b_i$ be defined as in the LSTD description in \eqrefb{LSTD:defAb}.
For clarity, we restate their definition along with their recursive writing:
\begin{align}
  A_i &= \sum_{j=1}^i z_j \Delta\phi_{j}^T = A_{i-1} + z_i
  \Delta\phi_{i+1}^T
  \\
  b_i &= \sum_{j=1}^i z_j \rho_j r_j = b_{i-1} + z_i \rho_i r_i.
\end{align}
Then, it can be seen that the LSPE($\lambda$) update is:
\begin{equation}
  \theta_i = \theta_{i-1} + N_i(b_i - A_i \theta_{i-1}).
\end{equation}
The overall computation is provided in Algorithm~\ref{algo:lspe}.

\begin{algorithm2e}[tbh]
  \SetAlgoVlined
  \caption{Off-policy LSPE($\lambda$)}
  \label{algo:lspe}
  \BlankLine
  {\textbf{Initialization}}\;
  Initialize vector $\theta_0$ and matrix $N_0$ \;
  Set $z_0=0$, $A_0 = 0$ and $b_0 = 0$\;
  \BlankLine
  \For{$i=1,2,\dots$}{
    \BlankLine
    \textbf{{Observe}} $\phi_i, r_i, \phi_{i+1}$\;
    \BlankLine
    {\textbf{Update traces}} \;
    $
    z_i = \gamma\lambda \rho_{i-1} z_{i-1} + \phi_i
    $ \;
    \BlankLine
    {\textbf{Update parameters}} \;
    $
    N_i = N_{i-1} - \frac{N_{i-1}\phi_i\phi_i^T N_{i-1}}{1+\phi_i^T N_{i-1} \phi_i}
    $ \;
    $
    A_i = A_{i-1} + z_i\Delta\phi_i^T
    $\;
    $
    b_i = b_{i-1} + \rho_i z_i r_i
    $\;
    $\theta_i = \theta_{i-1} + N_i (b_i - A_i \theta_{i-1})$ \;
  }
\end{algorithm2e}

This algorithm, (briefly) mentioned by \cite{Yu:2010}, generalizes
the LSPE($\lambda$) algorithm of~\citet{ioffe} to off-policy
learning. With respect to LSTD($\lambda$), which computes
$\theta_i=(A_i)^{-1}b_i$ (\cf \eqrefb{LSTD:theta}) at each
iteration, LSPE($\lambda$) is fundamentally recursive (as it is
based on an iterated fixed-point search). Along with the almost sure
convergence of $\frac{1}{i}A_i$ and $\frac{1}{i}b_i$ to $A$ and $b$
(defined in \eqrefb{LSTD:Ab}), it can be shown that $i N_i$
converges to $N=(\Phi^T D_0 \Phi)^{-1}$ (see for instance
\cite{Nedic:2003}) so that, asymptotically, LSPE($\lambda$) behaves
as:
\begin{equation}
\theta_i =\theta_{i-1} + N ( b - A \theta_{i-1})
=N b+  (I- NA) \theta_{i-1}
\end{equation}
or using the defintion of $\Pi_0$, $A$, $b$ (\eqrefb{LSTD:Ab}) and $T^\lambda$ (\eqrefb{tlambda}):
\begin{equation}
\label{eq:LSPE_asympt}
V_{\theta_i} = \Phi \theta_i=  \Phi N b + \Phi (I-N A) \theta_{i-1} = \Pi_0 T^\lambda V_{\theta_{i-1}}.
\end{equation}
The behavior of this sequence depends on whether the spectral radius
of $\Pi_0 T^\lambda$  is smaller than $1$ or not. Thus, the analyses
of \cite{Yu:2010} and \cite{Nedic:2003} (for the convergence of
$N_i$) imply the following convergence result:
under the assumptions required for the
convergence of off-policy LSTD($\lambda$), and the additional
assumption that the operator $\Pi_0 T^\lambda$ has a spectral radius
smaller than 1 (so that it is contracting), LSPE($\lambda$) also
converges almost surely to the fixed-point of the compound $\Pi_0
T^\lambda$ operator.

There are two sufficient conditions that can ensure
such a desired contraction property. The first one is when one
considers on-policy learning, as \cite{Nedic:2003} did when they derived the first convergence proof of (on-policy) LSPE($\lambda$). When the behavior policy $\pi_0$ is different from the target
policy $\pi$, a sufficient condition for contraction is that $\lambda$
be close enough to 1; indeed, when $\lambda$ tends to 1, the spectral
radius of $T^\lambda$ tends to zero and can potentially balance an
expansion of the projection $\Pi_0$.  In the off-policy case, when
$\gamma$ is sufficiently big, a small value of $\lambda$ can make
$\Pi_0 T^\lambda$ expansive (see~\citet{vanroy:1997} for an example in
the case $\lambda=0$) and off-policy LSPE($\lambda$) will then
diverge.  Eventually, Equations \eqref{eq:projfixpoint} and
\eqref{eq:LSPE_asympt} show that when $\lambda=1$, both
LSTD($\lambda$) and LSPE($\lambda$) asymptotically coincide (as $T^1
V$ does not depend on $V$).

\subsection{Off-policy FPKF($\lambda$)}

\label{sec:fpkf}

The off-policy FPKF($\lambda$) algorithm corresponds to the instantiation $\red{\xi}
= \blue{\theta_{j-1}}$  in Problem~\eqref{eq:cost_lin_gen}:
\begin{equation}
  \theta_i = \argmin_{\omega\in\mathbb{R}^p} \sum_{j=1}^i
  (\phi_j^T \blue{\theta_{j-1}} + \sum_{k=j}^i \tilde{\rho}_j^{k-1}(\rho_k r_k - \Delta\phi_{k}^T \blue{\theta_{j-1}})
  - \phi_j^T\omega)^2.
\end{equation}
This can be solved by zeroing the gradient respectively to $\omega$:
\begin{equation}
  \theta_i = N_i \sum_{j=1}^i
  \phi_j(\phi_j^T \theta_{j-1} + \sum_{k=j}^i \tilde{\rho}_j^{k-1}(\rho_k r_k - \Delta\phi_{k}^T
  \theta_{j-1})),
\end{equation}
where $N_i$ is the matrix introduced for LSPE($\lambda$) in \eqrefb{LSPE:N}.
For clarity, we restate its definition here and its recursive writing:
\begin{equation}
\label{eq:FPKF:N}
N_i = (\sum_{j=1}^i \phi_j \phi_j^T)^{-1} = N_{i-1} -
  \frac{N_{i-1} \phi_i \phi_i^T N_{i-1}}{1 + \phi_i^T N_{i-1} \phi_i}.
\end{equation}
Using Lemma~\ref{lemma:sum}, one obtains:
\begin{equation}
  \theta_i = N_i (\sum_{j=1}^i
  \phi_j\phi_j^T \theta_{j-1} + \sum_{j=1}^i\sum_{k=1}^j \phi_k \tilde{\rho}_k^{j-1}(\rho_j r_j - \Delta\phi_{j}^T
  \theta_{k-1})).
\end{equation}
With respect to the previously described algorithms, the difficulty here is that on the right side there is a dependence with all the previous terms $\theta_{k-1}$ for $1 \leq k \leq i$. Using the symmetry of the dot product $\Delta\phi_{j}^T \theta_{k-1}=\theta_{k-1}^T
\Delta\phi_{j}$,  it is possible to write a recursive algorithm by  introducing the trace matrix $Z_j$ that integrates the subsequent values of $\theta_k$ as follows:
\begin{equation}
  Z_j = \sum_{k=1}^j \tilde{\rho}_k^{j-1} \phi_k \theta_{k-1}^T
  = Z_{j-1} + \gamma \lambda \rho_{j-1} \phi_j \theta_{j-1}^T.
\end{equation}
With this notation we obtain:
\begin{equation}
  \theta_i = N_i (\sum_{j=1}^i
  \phi_j\phi_j^T \theta_{j-1} + \sum_{j=1}^i (z_j \rho_j r_j - Z_j
  \Delta\phi_j)).
\end{equation}
Using \eqrefb{FPKF:N} and a few algebraic manipulations, we end up with:
\begin{equation}
  \theta_i = \theta_{i-1} + N_i(z_i\rho_i r_i - Z_i \Delta\phi_i).
\end{equation}
This is the parameter update as provided in Algorithm~\ref{algo:fpkf}.

\begin{algorithm2e}[tbh]
  \SetAlgoVlined
  \caption{Off-policy FPKF($\lambda$)}
  \label{algo:fpkf}
  \BlankLine
  {\textbf{Initialization}}\;
  Initialize vector $\theta_0$ and matrix $N_0$ \;
  Set $z_0=0$ and $Z_0 = 0$\;
  \BlankLine
  \For{$i=1,2,\dots$}{
    \BlankLine
    \textbf{{Observe}} $\phi_i, r_i, \phi_{i+1}$\;
    \BlankLine
    {\textbf{Update traces}} \;
    $
    z_i = \gamma\lambda \rho_{i-1} z_{i-1} + \phi_i
    $ \;
    $
    Z_i = \gamma\lambda \rho_{i-1} Z_{i-1} + \phi_i \theta_{i-1}^T
    $\;
    \BlankLine
    {\textbf{Update parameters}} \;
    $
    N_i = N_{i-1} - \frac{N_{i-1}\phi_i\phi_i^T N_{i-1}}{1+\phi_i^T N_{i-1} \phi_i}
    $ \;
    $\theta_i = \theta_{i-1} + N_i (z_i \rho_i r_i - Z_i \Delta\phi_i)$ \;
  }
\end{algorithm2e}

It generalizes the FPKF algorithm of~\citet{Choi:2006} that was
originally only introduced without traces and in the on-policy case.
As LSPE($\lambda$), this algorithm is fundamentally recursive.
However, its overall behavior is quite different. As we discussed
for LSPE($\lambda$), $i N_i$ can be shown to tend asymptotically to
$N=(\Phi^T D_0 \Phi)^{-1}$ and FPKF($\lambda$) iterates eventually
resemble:
\begin{equation}
\theta_i = \theta_{i-1} + \frac{1}{i}N(z_i\rho_i r_i - Z_i \Delta\phi_i).
\end{equation}
The term in brackets is a random component (that only depends on the last
transition) and $\frac{1}{i}$ acts as a learning coefficient that
asymptotically tends to 0. In other words, FPKF($\lambda$) has a
\emph{stochastic approximation} flavour. In particular, one can see
FPKF(0) as a stochastic approximation of LSPE(0). Indeed, asymptotically, FPKF(0)
  does the following update
$$\theta_i = \theta_{i-1} + \frac{1}{i}N(\rho_i \phi_i r_i -
  \phi_i \Delta\phi_i^T \theta_{i-1}),$$ and one can notice that $\rho_i
  \phi_i r_i$ and $\phi_i \Delta\phi_i^T$ are samples of $A$ and $b$
  to which $A_i$ and $b_i$ converge through LSPE(0).
When $\lambda>0$, the situation is less clear -- up to the fact that since $T^1 V$ does not depend on $V$,
we expect FPKF to asymptotically behave like LSTD and LSPE when $\lambda$ tends to $1$.

Due to its much more involved form (notably the matrix trace $Z_j$
integrating the values of all the values $\theta_k$ from the start),
it does not seem easy to provide a guarantee for FPKF($\lambda$), even
in the on-policy case. To our knowledge, there does not exist
any \emph{proof of convergence} for stochastic approximation
algorithms in the off-policy case with traces\footnote{An analysis of
TD($\lambda$), with a simplifying assumption that forces the algorithm
to stay bounded is given by \cite{Yu:2010}. An analysis of  GQ($\lambda$) is provided by \cite{Maei:2010}, with an
assumption on the second moment of the traces, which does not hold in
general (see Proposition 2 in \citep{Yu:2010}). A full analysis of these
algorithms thus remains to be done. See also Sections~\ref{subsec:gtd}
and~\ref{subsec:grad_pfp}.}, and a related result for FPKF($\lambda$)
thus seems difficult. Based on the above-mentioned relation between
FPKF(0) and LSPE(0) and the experiments we have run
(see \refsec{exp}), we conjecture that off-policy FPKF($\lambda$) has
the same asymptotic behavior as LSPE($\lambda$).
We leave the formal study of this algorithm for future work.

\subsection{Off-policy BRM($\lambda$)}

\label{sec:brm}

The off-policy BRM($\lambda$) algorithm corresponds to the instantiation $\red \xi
= \blue \omega$ in Problem~\eqref{eq:cost_lin_gen}:

\begin{equation}
  \theta_i = \argmin_{\omega\in\mathbb{R}^p} \sum_{j=1}^i
  (\phi_j^T \blue \omega + \sum_{k=j}^i \tilde{\rho}_j^{k-1}(\rho_k r_k - \Delta\phi_{k}^T \blue \omega)
  - \phi_j^T \omega)^2
  = \argmin_{\omega\in\mathbb{R}^p} \sum_{j=1}^i (\sum_{k=j}^i \tilde{\rho}_j^{k-1}(\rho_k r_k - \Delta\phi_{k}^T
  \omega))^2.
\end{equation}
Define
\begin{equation}
    \psi_{j\rightarrow i} = \sum_{k=j}^i \tilde{\rho}_j^{k-1} \Delta\phi_k
    \text{~and~}  z_{j\rightarrow i} = \sum_{k=j}^i \tilde{\rho}_j^{k-1} \rho_k r_k.
\end{equation}
This yields the following batch estimate:
\begin{equation}
        \label{eq:BRM:batchest}
  \theta_i = \argmin_{\omega\in\mathbb{R}^p} \sum_{j=1}^i (z_{j\rightarrow
  i} - \psi_{j\rightarrow i}^T\omega)^2
= (\tilde A_i)^{-1}\tilde b_i
\end{equation}
where
\begin{equation}
  \tilde A_i = \sum_{j=1}^i \psi_{j\rightarrow
  i} \psi_{j\rightarrow i}^T
\mbox{~~~and~~~}\tilde b_i =  \sum_{j=1}^i  \psi_{j\rightarrow
  i} z_{j\rightarrow i}.
\end{equation}
The transformation of this batch estimate into a recursive update rule is somewhat tedious (it involves three ``trace'' variables), and the details are deferred to Appendix~\ref{brm:derivation} for clarity. The resulting BRM($\lambda$) method is provided in Algorithm~\ref{algo:brm}. Note that at each step, this algorithm involves the inversion of a $2 \times 2$ matrix (involving the $2 \times 2$ identity matrix $I_2$), inversion that admits a straightforward analytical solution.  The
  computational complexity of an iteration of BRM($\lambda$) is thus $O(p^2)$ (as for the preceding least-squares-based algorithms). 

\begin{algorithm2e}[tbh]
  \SetAlgoVlined
  \caption{Off-policy BRM($\lambda$)}
  \label{algo:brm}
  \BlankLine
  {\textbf{Initialization}}\;
  Initialize vector $\theta_0$ and matrix $C_0$ \;
  Set $y_0=0$, $\mathfrak{D}_0=0$ and $z_0 = 0$\;
  \BlankLine
  \For{$i=1,2,\dots$}{
    \BlankLine
    \textbf{{Observe}} $\phi_i, r_i, \phi_{i+1}$\;
    \BlankLine
    {\textbf{Pre-update traces}} \;
    $
    y_i = (\gamma\lambda\rho_{i-1})^2 y_{i-1} + 1
    $ \;
    \BlankLine
    {\textbf{Compute}} \;
    $
   U_i = \begin{pmatrix}
      \sqrt{y_i} \Delta\phi_i + \frac{\gamma \lambda\rho_{i-1}}{\sqrt{y_i}}
      \mathfrak{D}_{i-1} & \frac{\gamma \lambda\rho_{i-1}}{\sqrt{y_i}}
      \mathfrak{D}_{i-1}
    \end{pmatrix}^T
    $ \;
    $V_i = \begin{pmatrix}
      \sqrt{y_i} \Delta\phi_i + \frac{\gamma \lambda\rho_{i-1}}{\sqrt{y_i}}
      \mathfrak{D}_{i-1} & -\frac{\gamma \lambda\rho_{i-1}}{\sqrt{y_i}}
      \mathfrak{D}_{i-1}
    \end{pmatrix}^T$ \;
    $
    W_i =\begin{pmatrix}
      \sqrt{y_i} \rho r_i + \frac{\gamma\lambda\rho_{i-1}}{\sqrt{y_i}} z_{i-1} & -
      \frac{\gamma\lambda\rho_{i-1}}{\sqrt{y_i}} z_{i-1}
    \end{pmatrix}^T
    $\;
    \BlankLine
    {\textbf{Update parameters}} \;
    $
    \theta_i = \theta_{i-1} +  C_{i-1} U_i \left(I_2 + V_i C_{i-1}
      U_i\right)^{-1}\left(W_i - V_i \theta_{i-1}\right)
    $ \;
    $C_i =  C_{i-1} - C_{i-1} U_i \left(I_2 + V_i C_{i-1} U_i\right)^{-1}
      V_i C_{i-1}$ \;
    \BlankLine
    {\textbf{Post-update traces}} \;
    $
    \mathfrak{D}_i = (\gamma\lambda \rho_{i-1}) \mathfrak{D}_{i-1} +
      \Delta\phi_i y_i
    $ \;
    $
    z_i = (\gamma\lambda \rho_{i-1}) z_{i-1} + r_i \rho_i y_i
    $ \;
  }
\end{algorithm2e}

GPTD and KTD, which are close to BRM, have also been extended with
some trace mechanism; however,
GPTD($\lambda$)~\citep{Engel:2005}\footnote{Technically,
GPTD($\lambda$) is not exactly a generalization of GPTD as it does not
reduce to it when $\lambda=0$. It is rather a variation.},
KTD($\lambda$)~\citep{Supelec637} and the just described
BRM($\lambda$) are different algorithms. Briefly, GPTD($\lambda$) is
very close to LSTD($\lambda$) and KTD($\lambda$) uses a
different Bellman operator\footnote{The corresponding loss is
$(\hat{T}^0_{j,i} \hat{V}(\omega) - \hat{V}_{\omega}(s_j) + \gamma
\lambda (\hat{T}^1_{j+1,i} \hat{V}(\omega) -
\hat{V}_{\omega}(s_{j+1})))^2$. With $\lambda=0$ it gives
$\hat{T}^0_{j,i}$ and with $\lambda = 1$ it provides
$\hat{T}^1_{j,i}$.}. As BRM($\lambda$) builds a linear system whose
solution is updated recursively, it resembles LSTD($\lambda$).
However, the system it builds is different. The following theorem, proved
in Appendix~\ref{proofbrm}, 
partially characterizes the behavior of BRM($\lambda$) and its potential
limit\footnote{Our proof is similar to that
of Proposition 4 in \cite{Bertsekas:2009}. 
The overall arguments are the following: \eqrefb{condBRM}
implies that the traces can be truncated at some depth $l$, whose
influence on the potential limit of the algorithm vanishes when $l$
tends to $\infty$. For all $l$, the $l$-truncated version of the
algorithm can easily be analyzed through the ergodic theorem for
Markov chains. Making $l$ tend to $\infty$ allows tying the
convergence of the original arguments to that of the truncated
version. Eventually, the formula for the limit of the truncated
algorithm is computed and one derives the limit.}.
\begin{theorem}\label{th2}
Assume that the stochastic matrix $P_0$ of the \emph{behavior} policy is irreducible and has stationary distribution $\mu_0$. Further assume that there exists a coefficient $\beta<1$ such that
\begin{equation}
\label{eq:condBRM}
\forall (s,a),~~\lambda\gamma\rho(s,a) \leq \beta,
\end{equation}
then $\frac{1}{i}\tilde A_i$ and $\frac{1}{i}\tilde b_i$ respectively converge almost surely to
\begin{align}
\tilde A& = \Phi^T\left[ D - \gamma D P - \gamma P^T D + \gamma^2 D' + S(I-\gamma P) + (I-\gamma P^T)S^T \right]\Phi\\
\tilde b& = \Phi^T \left[ (I-\gamma P^T)Q^T D + S \right] R^\pi
\end{align}
where we wrote:
\begin{align}
D&=\diag\left((I-(\lambda\gamma)^2 \tilde P^T)^{-1}\mu_0\right) &Q=(I-\lambda\gamma P)^{-1}\\
D'&= \diag\left(\tilde P^T(I-(\lambda\gamma)^2 \tilde P^T)^{-1}\mu_0\right)
&S=\lambda\gamma (DP-\gamma D')Q
\end{align}
and where $\tilde P$ is the matrix whose coordinates are $\tilde
p_{s s'}=\sum_a \pi(a|s)\rho(s,a)P(s'|s,a)$. Then, the
BRM($\lambda$) algorithm converges with probability 1 to $\tilde
A^{-1} \tilde b$.
\end{theorem}
The assumption given by \eqrefb{condBRM} trivially holds in the
on-policy case (in which $\rho(s,a)=1$ for all $(s,a)$) and in the
off-policy case when $\lambda \gamma$ is sufficiently small with
respect to the mismatch between policies. Note in particular that
this result implies the almost sure convergence of the GPTD/KTD
algorithms in the on-policy and no-trace case, a question that was
still open in the literature (see for instance the conclusion of
\citet{Engel:2005}). The matrix $\tilde P$, which is in general not
a stochastic matrix, can have a spectral radius bigger than~1;
\eqrefb{condBRM} ensures that $(\lambda\gamma)^2 \tilde P$ has a
spectral radius smaller than $\beta$ so that $D$ and $D'$ are well
defined. Removing assumption of \eqrefb{condBRM} does not seem easy,
since by tuning $\lambda\gamma$ maliciously, one may force the
spectral radius of $(\lambda\gamma)^2 \tilde P$ to be as close to 1
as one may want, which would make $\tilde A$ and $\tilde b$ diverge.
Though the quantity $\tilde A^{-1}\tilde b$ may compensate for these
divergences, our current proof technique cannot account for this
situation and a related analysis constitutes possible future work.

The fundamental idea behind the Bellman Residual approach is to
address the computation of the fixed-point of $T^\lambda$ differently
from the previous methods. Instead of computing the projected fixed-point as in \eqrefb{projfixpoint}, one considers the following overdetermined
system
\begin{align}
\Phi\theta &\simeq  T^\lambda \Phi \theta \\
\Leftrightarrow ~~\Phi\theta& \simeq (I-\lambda \gamma P)^{-1}(R+(1-\lambda)\gamma P \Phi \theta) &  \mbox{(\eqrefb{tlambda})}\\
\Leftrightarrow ~~\Phi\theta& \simeq  Q R + (1-\lambda)\gamma P Q \Phi \theta \\
\Leftrightarrow ~~\Psi \theta&\simeq QR
\end{align}
with $\Psi=\Phi-(1-\lambda)\gamma P Q \Phi$, and solves it in a
least-squares sense, that is by computing $\theta^* =
\bar{A}^{-1}\bar{b}$ with $\bar{A}=\Psi^T \Psi$ and $\bar b=\Psi^T Q
R$. One of the motivations for this approach is that, as opposed to
the matrix $A$ of LSTD/LSPE/FPKF, $\bar{A}$ is invertible for all
values of $\lambda$, and one can always guarantee a finite error
bound with respect to the best projection~(see
\cite{Schoknecht:2002,yu:2008,scherrer:2010}). If the goal of
BRM($\lambda$) is to compute $\bar A$ and $\bar b$ from samples,
what it actually computes ($\tilde A$ and $\tilde b$ as characterized in Theorem~\ref{th2}) will in
general be biased because the estimation is based on a single
trajectory\footnote{It is
  possible to remove the bias when $\lambda=0$ by using double
  samples. However, in the case where $\lambda>0$, the possibility to
  remove the bias seems much more difficult.
}. Such a
bias adds an uncontrolled variance term to $\bar A$ and $\bar b$
(\eg, see \cite{antos2006}); an interesting consequence is that
$\tilde A$ is always non-singular\footnote{$\bar A$ is by construction
  positive definite, and $\tilde A$ equals $\bar A$ plus a positive
  term (the variance term), and is thus also positive definite.}. More
precisely, there are two sources of bias in the estimation: one
results from the non Monte-carlo evaluation (the fact that
$\lambda<1$) and the other from the use of the correlated importance
sampling factors (as soon as one considers off-policy learning). The
interested reader may check that in the on-policy case, and when
$\lambda$ tends to~1, $\tilde A$ and $\tilde b$ coincide with $\bar
A$ and $\bar b$. However, in the strictly off-policy case, taking
$\lambda=1$ does not prevent the bias due to the correlated
importance sampling factors. If we have argued that LSTD/LSPE/FPKF
should asymptotically coincide when $\lambda=1$, we see here that BRM should
generally differ in an off-policy situation.

\section{Stochastic gradient based extensions to eligibility traces and off-policy learning}
\label{sec:algo_grad}

We have just provided a systematic derivation of all
least-squares-based algorithms for learning with eligibility traces
in an off-policy manner. When the number of features $p$ is very
large, the $O(p^2)$ complexity involved by a least-squares approach may
be prohibitive. In such a situation, a natural alternative is to consider
an approach based on a stochastic gradient descent of the objective
function of interest~\citep{bottou-bousquet-2011,Sutton:2009a,Maei:2010}.

In this section, we will describe a systematic derivation of
stochastic gradient based algorithms for learning in an off-policy
manner with eligibility traces. The principle followed is the
same as for the least-squares-based approaches: we shall instantiate the
algorithmic pattern of \eqrefb{cost_function_generic} by choosing
the value of $\xi$ and update the parameter so as move towards the minimum
of $J(\theta_i,\xi)$ in \eqrefb{cost_function_generic} using a stochastic gradient
descent.
To make the pattern of \eqrefb{cost_function_generic} precise,  we need to define
the empirical approximate operator we use. We will consider the \emph{untruncated} $\hat{T}^\lambda_{i,n}$ operators
(written in the followings $\hat{T}^\lambda_i$, with a slight abuse
of notation):
\begin{equation}
\label{eq:defgradop}
  \hat{T}_i^\lambda V = V(s_i) + \sum_{j=i}^n
  (\gamma\lambda)^{j-i} \left(\rho_i^j \hat{T}_j V - \rho_i^{j-1}
  V(s_j)\right)
\end{equation}
where $n$ is the total length of the trajectory.

It should be noted that algorithmic derivations will here be a little
bit more involved than in the least-squares case. First, with the
instantiation $\xi=\theta_i$, the pattern given
in \eqrefb{cost_function_generic} is actually a fixed-point problem
onto which one cannot directly perform a stochastic gradient descent
(this issue will be addressed in
\refsubsec{grad_pfp} through the introduction of an auxiliary objective
function, following the approach originally proposed by
\cite{Sutton:2009a}). 
%
A second difficulty is the following: the just introduced empirical
operator $\hat{T}^\lambda_i$ depends on all the trajectory after step
$i$ (on the future of the process), and is for this reason usually
coined a \emph{forward view} estimate. Though it would be possible, in
principle, to implement a gradient descent based on this \emph{forward view},
it would not be very memory nor time efficient. Thus, we will follow
a usual trick of the literature by deriving recursive algorithms based
on a \emph{backward view} estimate that is equivalent to the \emph{forward view}
in expectation. To do so, we will repeatedly use the following identity that highlights the fact that the estimate $\hat{T}^\lambda_i V$ can be written as
a forward recursion:
%
\begin{lemma}
  \label{lemma:op_fw_rec}
  Let $\hat{T}_i^\lambda$ be the 
  operator defined in \eqrefb{defgradop} and let $V\in\mathbb{R}^S$. We have
  \begin{equation}
    \hat{T}_i^\lambda V = \rho_i r_i + \gamma \rho_i (1-\lambda)    V(s_{i+1}) + \gamma\lambda \rho_i \hat{T}_{i+1}^\lambda V.
  \end{equation}
\end{lemma}
\begin{proof}
  Using notably the identity $\rho_i^j = \rho_ i \rho_{i+1}^j$, we
  have:
  \begin{align}
  \hat{T}_{i}^\lambda V &= V(s_i) + \sum_{j=i}^n
  (\gamma\lambda)^{j-i} \left(\rho_i^j \hat{T}_j V - \rho_i^{j-1}
  V(s_j)\right)
  \\
  &= V(s_i) + \rho_i \hat{T}_i V - V(s_i) + \gamma\lambda \rho_i
  \sum_{j=i+1}^n (\rho_i^j\hat{T}_j V - \rho_i^{j-1} V(s_j))
  \\
  &= \rho_i \hat{T}_i V + \gamma\lambda\rho_i \left(\hat{T}_{i+1}^\lambda V - V(s_{i+1})
  \right).
  \end{align}
~\vspace{-1.5cm}

\end{proof}

To sum up, the ``recipe'' that we are about to use to derive off-policy gradient learning
algorithms based on eligibility traces will consist of the following steps:
\begin{enumerate}
  \item write the empirical generic cost
  function~\eqref{eq:cost_function_generic} with the untruncated
  Bellman operator of \eqrefb{defgradop}
  %
  \item instantiate $\xi$ and derive the gradient-based update rule
  (with some additional work for $\xi=\theta_{i}$, see
  Section~\ref{subsec:grad_pfp});
  \item turn the \emph{forward view} into an equivalent (in
  expectation) \emph{backward view}.
\end{enumerate}
The next subsection details the precise derivation of the algorithms.

\subsection{Off-policy TD($\lambda$)}
\label{subsec:gtd}

Because it is the simplest, we begin by considering the bootstrap approach, that is the instantiation
$\xi = \theta_{j-1}$. The cost function to be minimized is
therefore:
\begin{equation}
  \sum_{j=1}^i \left(\hat{T}_j^\lambda
  \hat{V}_{\theta_{j-1}} - \hat{V}_\omega(s_j)\right)^2.
\end{equation}
Minimized with a stochastic gradient descent, the related update
rule is ($\alpha_i$ being a standard learning rate and recalling
that $\hat{V}_\omega(s_i) = \omega^T \phi(s_i) = \omega^T \phi_i$):
\begin{align}
  \theta_i &= \theta_{i-1} - \frac{\alpha_i}{2}\nabla_{\omega}  \left(\hat{T}_i^\lambda
  \hat{V}_{\theta_{i-1}} -
  \hat{V}_\omega(s_i)\right)^2\Big|_{\omega=\theta_{i-1}}
  \\
  &= \theta_{i-1} + \alpha_i \phi_i \left(\hat{T}_i^\lambda \hat{V}_{\theta_{i-1}} -
  \hat{V}_{\theta_{i-1}}(s_i)\right).
  \label{eq:update_td_forward}
\end{align}
At this point, one could notice that the exact same update rule
would have been obtained with the instantiation $\xi =
\theta_{i-1}$. This was to be expected: as only the last term of the
sum is considered for the update, we have $j=i$, therefore $\xi =
\theta_{i-1} = \theta_{j-1}$.

\eqrefb{update_td_forward} makes use of a $\lambda$-TD error
defined as
\begin{equation}
  \delta_i^\lambda(\omega) = \hat{T}_i^\lambda \hat{V}_\omega -
  \hat{V}_\omega(s_i).
\end{equation}
For convenience, let also $\delta_i$ be the standard (off-policy) TD
error defined as
\begin{equation}
  \delta_i(\omega)  = \delta^{\lambda=0}_i(\omega) = \rho_i \hat{T}_i
  \hat{V}_{\omega} - \hat{V}_{\omega}(s_i) = \rho_i\left(r_i + \gamma
  \hat{V}_{\omega}(s_{i+1})\right) - \hat{V}_{\omega}(s_i).
\end{equation}
The $\lambda$-TD error can be expressed as a forward recursion:
\begin{lemma}
  \label{lemma:td_fwd_rec}
  Let $\delta_i^\lambda$ be the $\lambda$-TD error and
  $\delta_i$ be the standard TD error.  Then for all $\omega$,
  \begin{equation}
    \delta_i^\lambda(\omega) = \delta_i(\omega) + \gamma \lambda \rho_i
    \delta_{i+1}^\lambda(\omega).
  \end{equation}
\end{lemma}
\begin{proof}
  This is a corollary of
  Lemma~\ref{lemma:op_fw_rec}:
  \begin{align}
&    \hat{T}_i^\lambda V_{\omega} ~=~ \rho_i r_i + \gamma \rho_i (1-\lambda)
    V_{\omega}(s_{i+1}) + \gamma\lambda \rho_i \hat{T}_{i+1}^\lambda V_{\omega}
    \\
    \Leftrightarrow~~~&
    \hat{T}_i^\lambda V_{\omega} - V_{\omega}(s_i) ~=~ \rho_i r_i + \gamma \rho_i
    V_{\omega}(s_{i+1}) - V_{\omega}(s_i) + \gamma\lambda \rho_i (\hat{T}_{i+1}^\lambda V_{\omega} -
    V_{\omega}(s_{i+1}))
    \\
    \Leftrightarrow~~~&
    \delta_i^\lambda(\omega) ~=~ \delta_i(\omega) + \gamma \lambda \rho_i
    \delta_{i+1}^\lambda(\omega).
  \end{align}
~\vspace{-1.2cm}

\end{proof}
Therefore, we get the following update rule
\begin{equation}
  \theta_i = \theta_{i-1} + \alpha_i \phi_i
  \delta_i^\lambda(\theta_{i-1})
\end{equation}
with $\delta_i^\lambda(\theta_{i-1}) = \delta_i(\theta_{i-1}) + \gamma \lambda
\delta_{i+1}^\lambda(\theta_{i-1})$.
The key idea here is to find some backward recursion such
that in expectation, when the Markov chain has reached its steady
state, 
 it
provides the same result as the forward recursion.
Such a backward
recursion is given by the following lemma.
\begin{proposition}
  \label{prop:phi_delta}
  Let $z_i$ be the eligibility vector, defined by the following
  recursion:
  \begin{equation}
    z_i = \phi_i + \gamma\lambda \rho_{i-1}z_{i-1}.
  \end{equation}
  For all $\omega$, we have
  \begin{equation}
    E_{\mu_0}[\phi_i \delta_{i}^\lambda(\omega)] = E_{\mu_0}[z_i    \delta_{i}(\omega)].
  \end{equation}
\end{proposition}
\begin{proof}
For clarity, we omit the dependence with respect to $\omega$ and write below
$\delta_i$ (resp. $\delta_i^\lambda)$ for $\delta_i(\omega)$ (resp. $\delta_i^\lambda(\omega))$.
  The result relies on successive applications of
  Lemma~\ref{lemma:td_fwd_rec}. We have:
  \begin{align}
    E_{\mu_0}[\phi_i \delta_i^\lambda] &=
    E_{\mu_0}[\phi_i (\delta_i + \gamma \lambda \rho_i \delta_{i+1}^\lambda)]
    \\
    &=E_{\mu_0}[\phi_i \delta_i] + E_{\mu_0}[\phi_i \gamma \lambda \rho_i
    \delta_{i+1}^\lambda].
  \end{align}
  Moreover, we have that $E_{\mu_0}[\phi_i\rho_i \delta_{i+1}^\lambda] = E_{\mu_0}[\phi_{i-1} \rho_{i-1} \delta_{i}^\lambda]$,
  as expectation is done according to the stationary distribution,
 therefore:
  \begin{align}
    E_{\mu_0}[\phi_i \delta_i^\lambda] &= E_{\mu_0}[\phi_i \delta_i] + \gamma
    \lambda E_{\mu_0}[\phi_{i-1}
    \rho_{i-1} \delta_{i}^\lambda]
    \\
    &= E_{\mu_0}[\phi_i \delta_i] + \gamma
    \lambda E_{\mu_0}[\phi_{i-1}
    \rho_{i-1} (\delta_i + \gamma \lambda \rho_i \delta_{i+1}^\lambda)]
    \\
    &= E_{\mu_0}[\delta_i(\phi_i +
    \gamma\lambda\rho_{i-1}\phi_{i-1} +
    (\gamma\lambda)^2\rho_{i-1}\rho_{i-2}\phi_{i-2} + \dots)]
    \\
    &= E_{\mu_0}[\delta_i z_i].
  \end{align}
~\vspace{-1.2cm}

\end{proof}
This suggests to replace
\eqrefb{update_td_forward} by the following
update rule,
\begin{equation}
  \theta_i = \theta_{i-1} + \alpha_i z_i \delta_i(\theta_{i-1}),
\end{equation}
which is equivalent in expectation when the Markov chain 
has reached its steady state. This is
summarized in Algorithm~\ref{algo:td}.
\begin{algorithm2e}[tbh]
  \SetAlgoVlined
  \caption{Off-policy TD($\lambda$)}
  \label{algo:td}
  \BlankLine
  {\textbf{Initialization}}\;
  Initialize vector $\theta_0$\;
  Set $z_0=0$\;
  \BlankLine
  \For{$i=1,2,\dots$}{
    \BlankLine
    \textbf{{Observe}} $\phi_i, r_i, \phi_{i+1}$ \;
    \BlankLine
    {\textbf{Update traces}} \;
    $
    z_i = \gamma\lambda \rho_{i-1} z_{i-1} + \phi_i
    $ \;
    \BlankLine
    {\textbf{Update parameters}} \;
    $\theta_i = \theta_{i-1} + \alpha_i z_i (\rho_i r_i - \Delta\phi_i^T \theta_{i-1})$ \;
  }
\end{algorithm2e}

This algorithm has first been proposed in the tabular case
by~\citet{Precup:2000}. An off-policy TD($\lambda$) algorithm (with
function approximation) has been proposed by~\citet{Precup:2001},
but it differs significantly from the algorithm just described
(notably it differs in the definition of the traces and the
projected Bellman equation, and in the fact that it is constrained
to episodic trajectories). Algorithm~\ref{algo:td} has
actually first been proposed much more recently
by~\cite{bertsekas:09proj}.

An important issue for the analysis of this algorithm is the fact that
the trace $z_i$ may have an infinite variance, due to importance
sampling (see~\citet[Sec.~3.1]{yu:2010tech}). As far as we know, the
only existing analysis of off-policy TD($\lambda$) (as provided in
Algorithm~\ref{algo:td}) uses an additional contraint which forces the
parameters to be bounded: after each parameter update, the resulting
parameter vector is projected onto some predefined compact set. This
analysis is performed by~\citet[Sec.~4.1]{yu:2010tech}. Under the
standard assumptions of stochastic approximations and most of the
assumptions required for the on-policy TD($\lambda$) algorithm,
assuming moreover that $\Pi_0 T^\lambda$ is a contraction (which we
recall to hold for a big enough $\lambda$) and that the predefined
compact set used to project the parameter vector is a large enough
ball containing the fixed point of $\Pi_0 T^\lambda$, the constrained
version of off-policy TD($\lambda$) converges to this fixed-point
(therefore, the same solution as off-policy LSTD($\lambda$),
LSPE($\lambda$) and FPKF($\lambda$)). We refer
to~\citet[Sec.~4.1]{yu:2010tech} for further details.  An analysis of
the unconstrained version of off-policy TD($\lambda$) described in
Algorithm~\ref{algo:td} is an interesting topic for future research.

\subsection{Off-policy TDC($\lambda$) and off-policy GTD2($\lambda$)} \label{subsec:grad_pfp}

In this section, the case $\xi = \theta_i$ is considered. 
Following the general pattern, at step $i$, we would like to come up with a new parameter $\theta_i$ that moves (from $\theta_{i-1}$) closer to the minimum of the function
\begin{equation}
\omega \mapsto  J(\omega,\theta_i)= \left(\hat{T}_j^\lambda  \hat{V}_{\theta_i} - \hat{V}_\omega(s_j)\right)^2.  \label{eq:pfp_grad_1}
\end{equation}
This problem is tricky since the function to minimize contains what we want to compute -- $\theta_i$ -- as a parameter. For this reason we cannot directly perform a stochastic gradient descent of the right hand side.
Instead, we will consider an alternative (but equivalent)
formulation of the projected fixed-point minimization
$\theta=\arg\min_\omega \|V_\omega - \Pi_0 T^{\lambda} V_\omega\|^2$,
and will move from $\theta_{i-1}$ to $\theta_i$ by making one step of gradient descent of an estimate of the function 
$$ 
\theta \mapsto \|V_\theta - \Pi_0 T^{\lambda} V_\theta\|^2.
$$
With the following vectorial notations:
\begin{align}
  \mathbf{\hat{V}}_\omega &= \begin{pmatrix}
    \hat{V}_\omega(s_1) & \dots & \hat{V}_\omega(s_i)
  \end{pmatrix}^T,
  \\
  \mathbf{\hat{T}}^\lambda \mathbf{\hat{V}}_\omega &=
  \begin{pmatrix}
    \hat{T}^\lambda_1 \hat{V}_\omega & \dots & \hat{T}^\lambda_i
    \hat{V}_\omega
  \end{pmatrix}^T,
  \\
  \tphi &= \begin{bmatrix}
    \phi(s_1) & \dots & \phi(s_i)
  \end{bmatrix}^T,
  \\
  \tilde{\Pi}_0 &= \tphi (\tphi^T \tphi)^{-1} \tphi^T,
\end{align}
we consider the following objective function:
\begin{align}
  J(\theta) &= \left\|\mathbf{\hat{V}}_\theta -
  \tilde{\Pi}_0 \mathbf{\hat{T}}^\lambda \mathbf{\hat{V}}_\theta\right\|^2
  \\
  &= \left(\mathbf{\hat{V}}_\theta -
  \mathbf{\hat{T}}^\lambda
  \mathbf{\hat{V}}_\theta\right)^T \tilde{\Pi}_0 \left(\mathbf{\hat{V}}_\theta -
  \mathbf{\hat{T}}^\lambda
  \mathbf{\hat{V}}_\theta\right)
  \\
  &= \left(\sum_{j=1}^i \delta_j^\lambda(\theta)\phi_j\right)^T
  \left(\sum_{j=1}^i \phi_j \phi_j^T\right)^{-1}
  \left(\sum_{j=1}^i \delta_j^\lambda(\theta)\phi_j\right).
\end{align}
This is the derivation followed by~\cite{Sutton:2009a} in the case
$\lambda=0$ and by~\cite{Maei:2010} in the case $\lambda>0$ (and
off-policy learning). 
%
Let us introduce the following notation:
\begin{align}
\label{eq:defg}
g_j^\lambda = \nabla \hat{T}_j^\lambda \hat{V}_\theta.
\end{align} 
Note that since we consider a linear approximation this quantity does not depend on $\theta$. Noticing that $\nabla \delta_j^\lambda(\theta) = \phi_j - g_j^\lambda$, we can compute $\nabla J(\theta)$:
\begin{align}
  -\frac{1}{2}\nabla J(\theta) &= -\frac{1}{2}\nabla \left(\sum_{j=1}^i
  \delta_j^\lambda(\theta)\phi_j\right)^T
  \left(\sum_{j=1}^i \phi_j \phi_j^T\right)^{-1}
  \left(\sum_{j=1}^i \delta_j^\lambda(\theta)\phi_j\right)
  \\
  &= - \left(\nabla \sum_{j=1}^i
  \delta_j^\lambda(\theta)\phi_j\right)^T
  \left(\sum_{j=1}^i \phi_j \phi_j^T\right)^{-1}
  \left(\sum_{j=1}^i \delta_j^\lambda(\theta)\phi_j\right)
  \\
  &= \left(\sum_{j=1}^i (\phi_j - g_j^\lambda)\phi_j^T\right)
  \left(\sum_{j=1}^i \phi_j \phi_j^T\right)^{-1}
  \left(\sum_{j=1}^i \delta_j^\lambda(\theta)\phi_j\right)
  \label{eq:ankle1}
  \\
  &= \left(\sum_{j=1}^i
  \delta_j^\lambda(\theta)\phi_j\right) - \left(\sum_{j=1}^i g_j^\lambda \phi_j^T\right)
  \left(\sum_{j=1}^i \phi_j \phi_j^T\right)^{-1}
  \left(\sum_{j=1}^i \delta_j^\lambda(\theta)\phi_j\right).
\end{align}

Let $w_i(\theta)$ be a quasi-stationary estimate of the last part,
that can be recognized as the solution of a least-squares problem
(regression of $\lambda$-TD errors $\delta^\lambda_j$ on features $\phi_j$):
\begin{equation}
  w_i(\theta) \approx \left(\sum_{j=1}^i \phi_j \phi_j^T\right)^{-1}
  \left(\sum_{j=1}^i \delta_j^\lambda(\theta)\phi_j\right) = \argmin_\omega \sum_{j=1}^i \left(\phi_j^T \omega - \delta_j^\lambda(\theta)\right)^2.
\end{equation}
The identification with the above least-squares solution suggests to
use the following stochastic gradient descent to form the
quasi-stationary estimate:
\begin{equation}
  w_i = w_{i-1} + \beta_i \phi_i \left(\delta_i^\lambda(\theta_{i-1}) - \phi_i^T w_{i-1}\right).
  \label{eq:update_quasti_stat_estimate_w}
\end{equation}
This update rule makes use of the $\lambda$-TD error, defined
through a \emph{forward view}. 
As for the previous algorithm, we can use 
Proposition~\ref{prop:phi_delta} to obtain the following \emph{backward view}  update rule that is equivalent (in
expectation when the Markov chain reaches its steady state):
\begin{equation}
  w_i = w_{i-1} + \beta_i  \left(z_i \delta_i(\theta_{i-1}) - \phi_i (\phi_i^T w_{i-1})\right).
  \label{eq:update_quasti_stat_estimate_w_backward}
\end{equation}
Using this quasi-stationary estimate, the gradient can be
approximated as:
\begin{equation}
  -\frac{1}{2}\nabla J(\theta) \approx \left(\sum_{j=1}^i
  \delta_j^\lambda(\theta)\phi_j\right) - \left(\sum_{j=1}^i g_j^\lambda \phi_j^T\right)
  w_i.
\end{equation}
Therefore, a stochastic gradient descent gives the following update
rule for the parameter vector $\theta$:
\begin{equation}
  \theta_i = \theta_{i-1} + \alpha_i \left(\delta_i^\lambda(\theta_{i-1}) \phi_i -
  g_i^\lambda \phi_i^T w_i\right).
  \label{eq:update_gq_forward}
\end{equation}
Once again the \emph{forward view} term $\delta_i^\lambda(\theta_{i-1}) \phi_i$ can be
turned into a \emph{backward view} by using Proposition~\ref{prop:phi_delta}. There remains to
work on the term $g_i^\lambda \phi_i^T$.

First, one can notice that the term $g_i^\lambda$ satisfies a forward recursion.
\begin{lemma}
  \label{lemma:gradTrec}
We have
  \begin{equation}
    g_i^\lambda = \gamma \rho_i (1-\lambda)\phi_{i+1} + \gamma\lambda \rho_i g_{i+1}^\lambda.
  \end{equation}
\end{lemma}
\begin{proof}
  This result is simply obtained by applying the gradient to the forward recursion of $\hat{T}_i^\lambda V_\theta$ provided in Lemma~\ref{lemma:op_fw_rec} (according to $\theta$).
\end{proof}
Using this, the term $g_i^\lambda \phi_i^T$ can be worked out similarly to the
term $\delta_i^\lambda(\theta_{i-1})\phi_i$.
\begin{proposition}
  \label{prop:g_phi}
  Let $z_i$ be the eligibility vector defined in Proposition~\ref{prop:phi_delta}. We have
  \begin{equation}
    E_{\mu_0}[g_i^\lambda \phi_i^T] = E_{\mu_0}[\gamma \rho_i
    (1-\lambda) \phi_{i+1} z_i^T].
  \end{equation}
\end{proposition}
\begin{proof}
  The proof is similar to that of Proposition~\ref{prop:phi_delta}. 
  Writing $b_i = \gamma \rho_i (1-\lambda) \phi_{i+1}$ and $\eta_i =
  \gamma \lambda \rho_i$, we have
  \begin{align}
    E_{\mu_0}[g_i^\lambda \phi_i^T] &= E_{\mu_0}[(b_i + \eta_i
    g_{i+1}^\lambda)\phi_i^T]
    \\
    &= E_{\mu_0}[b_i \phi_i^T] + E_{\mu_0}[\eta_{i-1}(b_i + \eta_i
    g_{i+1}^\lambda)\phi_{i-1}^T]
    \\
    &= E_{\mu_0}[b_i z_i^T].
  \end{align}
~\vspace{-1.3cm}

\end{proof}
Using this result and Proposition~\ref{prop:phi_delta}, it is natural to replace
\eqrefb{update_gq_forward} by an update based on a backward
recursion:
\begin{equation}
  \theta_i = \theta_{i-1} + \alpha_i \left(z_i \delta_i - \gamma
  \rho_i(1-\lambda)\phi_{i+1}(z_i^T w_{i-1})\right).
  \label{eq:update_gq_backward}
\end{equation}

Last but not least, for the estimate $w_i$ to be indeed
quasi-stationary, the learning rates should satisfy the following
condition (in addition to the classical conditions):
\begin{equation}
  \lim_{i\rightarrow \infty} \frac{\alpha_i}{\beta_i} = 0.
\end{equation}
Eqs.~\eqref{eq:update_gq_backward}
and~\eqref{eq:update_quasti_stat_estimate_w_backward} define the
off-policy TDC($\lambda$) algorithm, summarized in
Algorithm~\ref{algo:tdc}. It was originally proposed
by~\cite{Maei:2010} under the name GQ($\lambda$).
We call it off-policy
TDC($\lambda$) to highlight the fact that it is the extension of the
original TDC algorithm of~\cite{Sutton:2009a} to off-policy learning
with traces. One can observe -- to our knowledge, this was never 
mentionned in the literature before -- that when $\lambda=1$, the learning rule
of TDC(1) reduces to that of TD(1).

\citet{Maei:2010}  show that the algorithm converges with probability 1 to the same
solution as the LSTD($\lambda$) algorithm (that is, to $\theta^* =
A^{-1} b$) under some technical assumptions. Contrary to off-policy
TD($\lambda$), this algorithm does not requires $\Pi_0 T^\lambda$ to
be a contraction in order to be convergent. Unfortunately, one of the
assumptions made in the analysis, requiring that the traces $z_i$ have
uniformly bounded second moments, is restrictive since in an
off-policy setting the traces $z_i$ may easily have an infinite
variance (unless the behavior policy is really close to the target
policy), as noted by~\cite{Yu:2010} (see also~\citet{Randhawa}).
A full proof of convergence thus still remains to be done.
\begin{algorithm2e}[tbh]
  \SetAlgoVlined
  \caption{Off-policy TDC($\lambda$), also known as GQ($\lambda$)}
  \label{algo:tdc}
  \BlankLine
  {\textbf{Initialization}}\;
  Initialize vector $\theta_0$ and $w_0$\;
  Set $z_0=0$\;
  \BlankLine
  \For{$i=1,2,\dots$}{
    \BlankLine
    \textbf{{Observe}} $\phi_i, r_i, \phi_{i+1}$ \;
    \BlankLine
    {\textbf{Update traces}} \;
    $
    z_i = \gamma\lambda \rho_{i-1} z_{i-1} + \phi_i
    $ \;
    \BlankLine
    {\textbf{Update parameters}} \;
    $\theta_i = \theta_{i-1} + \alpha_i \left( z_i (\rho_i r_i - \Delta\phi_i^T \theta_{i-1})
    - \gamma \rho_i (1-\lambda) \phi_{i+1} (z_i^T w_{i-1})\right)$ \;
    $ w_i = w_{i-1} + \beta_i \left( z_i (\rho_i r_i - \Delta\phi_i^T
    \theta_{i}) - \phi_i (\phi_i^T w_{i-1})\right)
    $ \;
  }
\end{algorithm2e}

Using the same principle (that is, performing a stochastic gradient
descent to minimize $J(\theta)$), in the $\lambda=0$ case, an
alternative to TDC, the GTD2 algorithm was derived by \citet{Sutton:2009a}. As far as we
know, it has never been extended to off-policy learning with traces;
we do it now. Notice that, given the derivation of GQ($\lambda$),
obtaining this algorithm is pretty straightforward.

To do so, we can start back from \eqrefb{ankle1}:
\begin{align}
  -\frac{1}{2}\nabla J(\theta) &=
  \left(\sum_{j=1}^i (\phi_j - g_j^\lambda)\phi_j^T\right)
  \left(\sum_{j=1}^i \phi_j \phi_j^T\right)^{-1}
  \left(\sum_{j=1}^i \delta_j^\lambda(\theta)\phi_j\right)
  \\
  &\approx \left(\sum_{j=1}^i (\phi_j -
  g_j^\lambda)\phi_j^T\right)w_i.
\end{align}
This suggests the following alternative update rule (based on
forward recursion):
\begin{equation}
  \theta_i = \theta_{i-1} + \alpha_i (\phi_i - g_i^\lambda)
  \phi_i^T w_i.
\end{equation}
Using Proposition~\ref{prop:g_phi}, it is natural to  use the following
alternative update rule, based on a backward recursion:
\begin{equation}
  \theta_i = \theta_{i-1} + \alpha_i \left(\phi_i (\phi_i^T w_{i-1}) - \gamma \rho_i (1-\lambda) \phi_{i+1} (z_i^T
  w_{i-1})\right).
\end{equation}
The update of $w_i$ remains the same, and put together it gives
off-policy GTD2($\lambda$), summarized in Algorithm~\ref{algo:gtd2}.
The analysis of this new algorithm constitutes a potential topic for future research.

\begin{algorithm2e}[tbh]
  \SetAlgoVlined
  \caption{Off-policy GTD2($\lambda$)}
  \label{algo:gtd2}
  \BlankLine
  {\textbf{Initialization}}\;
  Initialize vector $\theta_0$ and $w_0$\;
  Set $z_0=0$\;
  \BlankLine
  \For{$i=1,2,\dots$}{
    \BlankLine
    \textbf{{Observe}} $\phi_i, r_i, \phi_{i+1}$ \;
    \BlankLine
    {\textbf{Update traces}} \;
    $
    z_i = \gamma\lambda \rho_{i-1} z_{i-1} + \phi_i
    $ \;
    \BlankLine
    {\textbf{Update parameters}} \;
    $\theta_i = \theta_{i-1} + \alpha_i \left(\phi_i (\phi_i^T
    w_{i-1})
    - \gamma \rho_i (1-\lambda) \phi_{i+1} (z_i^T w_{i-1})\right)$ \;
    $ w_i = w_{i-1} + \beta_i \left( z_i (\rho_i r_i - \Delta\phi_i^T
    \theta_{i}) - \phi_i (\phi_i^T w_{i-1})\right)
    $ \;
  }
\end{algorithm2e}


\subsection{Off-policy gBRM($\lambda$)}
\label{subsec:gbrm}

The last considered approach is the residual approach, corresponding
to the instantiation $\xi = \omega$. The cost function to be
minimized is then:
\begin{equation}
  \sum_{j=1}^i \left(\hat{T}_j^\lambda \hat{V}_\omega -
  \hat{V}_\omega(s_j)\right)^2.
\end{equation}
Following the negative of the gradient of the last term leads to the following update rule:
\begin{align}
  \theta_i &= \theta_{i-1} - \alpha_i \nabla_\omega \left(\hat{T}_i^\lambda \hat{V}_\omega -
  \hat{V}_\omega(s_i)\right)^2 \Big|_{\omega=\theta_{i-1}}
  \\
  &= \theta_{i-1} - \alpha_i \nabla_\omega \left(\hat{T}_i^\lambda \hat{V}_\omega -
  \hat{V}_\omega(s_i)\right)\Big|_{\omega=\theta_{i-1}} \left(\hat{T}_i^\lambda \hat{V}_{\theta_{i-1}} -
  \hat{V}_{\theta_{i-1}}(s_i)\right)
  \\
  &= \theta_{i-1} + \alpha_i \left(\phi_i - g_i^\lambda\right) \delta_i^\lambda(\theta_{i-1}),
\end{align}
recalling the notation $g_i^\lambda = \nabla \hat{T}_i^\lambda
\hat{V}_\omega$ first defined in \eqrefb{defg}.

As usual, this update involves a \emph{forward view}, which we are going to turn into a \emph{backward view}.
The term $\phi_i
\delta_i^\lambda$ can be worked thanks to
Proposition~\ref{prop:phi_delta}. The term $g_i^\lambda \delta_i^\lambda$
is more difficult to handle, as it is the product of two \emph{forward
views} (until now, we only considered the product of a \emph{forward
view} with a non-recursive term). This can be done thanks to the following
original relation (the proof being somewhat tedious, it is deferred to Appendix~\ref{appbeurk}):
\begin{proposition}
  \label{prop:beurk}
  Write $g_i^\lambda = \nabla_\omega\hat{T}_i^\lambda$ and define
  \begin{align}
    c_i &= 1 + (\gamma\lambda \rho_{i-1})^2 c_{i-1},
    \\
    \zeta_i &= \gamma \rho_i(1-\lambda)\phi_{i+1}c_i +
    \gamma\lambda\rho_{i-1}\zeta_{i-1}
    \\ \text{and }
    d_i &= \delta_i c_i + \gamma\lambda \rho_{i-1} d_{i-1},
  \end{align}
  we have that
  \begin{equation}
    E_{\mu_0}[\delta_i^\lambda g_i^\lambda] =
    E_{\mu_0}[\delta_i \zeta_i + d_i \gamma
    \rho_i(1-\lambda)\phi_{i+1} - \delta_i \gamma \rho_i
    (1-\lambda) \phi_{i+1} c_i].
  \end{equation}
\end{proposition}

This result (together with Proposition~\ref{prop:phi_delta}) suggests to
update parameters as follows:
\begin{equation}
  \theta_i = \theta_{i-1} + \alpha_i \left(\delta_i(z_i + \gamma \rho_i
    (1-\lambda) \phi_{i+1} c_i -\zeta_i) - d_i \gamma \rho_i (1-\lambda)
  \phi_{i+1}\right).
\end{equation}
This gives the off-policy gBRM($\lambda$) algorithm, depicted in
Algorithm~\ref{algo:gbrm}. One can observe that gBRM(1) is equivalent to TD(1) (and thus also TDC(1), cf. the comment before the description of Algorithm~\ref{algo:tdc}).
\begin{algorithm2e}[tbh]
  \SetAlgoVlined
  \caption{Off-policy gBRM($\lambda$)}
  \label{algo:gbrm}
  \BlankLine
  {\textbf{Initialization}}\;
  Initialize vector $\theta_0$\;
  Set $z_0=0$, $d_0 = 0$, $c_0 = 0$, $\zeta_0 = 0$\;
  \BlankLine
  \For{$i=1,2,\dots$}{
    \BlankLine
    \textbf{{Observe}} $\phi_i, r_i, \phi_{i+1}$ \;
    \BlankLine
    {\textbf{Update traces}} \;
    $
    z_i =  \phi_i + \gamma\lambda \rho_{i-1} z_{i-1}
    $ \;
    $ c_i = 1 + (\gamma\lambda\rho_{i-1})^2 c_{i-1}$
    \;
    $ \zeta_i = \gamma \rho_i (1-\lambda) \phi_{i+1} c_i +
    \gamma\lambda \rho_{i-1} \zeta_{i-1} $
    \;
    $ d_i = (\rho_i r_i - \Delta\phi_i^T \theta_{i-1})c_i + \gamma\lambda \rho_{i-1} d_{i-1}$
    \;
    \BlankLine
    {\textbf{Update parameters}} \;
    $\theta_i = \theta_{i-1} + \alpha_i \left((\rho_i r_i - \Delta\phi_i^T \theta_{i-1})(z_i + \gamma \rho_i
    (1-\lambda) \phi_{i+1} c_i - \zeta_i)
    - d_i \gamma \rho_i(1-\lambda)\phi_{i+1}\right)$ \;
  }
\end{algorithm2e}
%
The analysis of this new algorithm is left for future research. 

\section{Empirical Study}
\label{sec:exp}

This section aims at empirically comparing the surveyed algorithms. As they only address the policy evaluation problem, we compare the algorithms in their ability to perform policy evaluation (no control, no policy optimization); however, they may straightforwardly be used in an approximate policy iteration approach \citep{Bertsekas:1996,Munos:2003}). In order to assess their quality, we consider finite problems where the exact value function can be computed.

More precisely, we consider Garnet problems~\citep{Archibald:95}, which are a class of randomly constructed finite MDPs. They do not correspond to any specific application, but are totally abstract while remaining representative of the kind of MDP that might be encountered in practice. In our experiments, a Garnet is parameterized by 4 parameters and is written $\mathcal{G}(n_S, n_A, b, p)$: $n_S$ is the number of states, $n_A$ is the number of actions, $b$ is a branching factor specifying how many possible next states are possible for each state-action pair ($b$ states are chosen uniformly at random and transition probabilities are set by sampling uniform random $b-1$ cut points between 0 and 1) and $p$ is the number of features (for function approximation). The reward is state-dependent: for a given randomly generated Garnet problem, the reward for each state is uniformly sampled between 0 and 1. Features are chosen randomly: $\Phi$ is a $n_S\times p$ feature matrix of which each component is randomly and uniformly sampled between 0 and 1, except the first row of which each component is set to 1 (this corresponds to a constant feature). The discount factor $\gamma$ is set to $0.95$ in all experiments.

We consider two types of problems, ``small'' and ``big'', respectively corresponding to instances $\mathcal{G}(30, 4, 2, 8)$ and $\mathcal{G}(100, 10, 3, 20)$. We also consider two types of learning: on-policy learning and off-policy learning. In the on-policy setting, for each Garnet a policy $\pi$ to be evaluated is randomly generated (by sampling randomly $n_A-1$ cut points between $0$ and $1$ for each state), and trajectories (to be used for learning) are sampled according to this same policy. In the off-policy setting, the policy $\pi$ to be evaluated is randomly generated the same way, but trajectories are sampled according to a uniform policy $\pi_0$ (that chooses each action with equal probability, that is $\pi_0(a|s) = \frac{1}{n_A}$ for any state-action couple).

For all algorithms, we choose $\theta_0 = 0$. For least-squares-based algorithms (LSTD, LSPE, FPKF and BRM), we set the initial matrices $(M_0, N_0, C_0)$ to $10^3 I$ (the higher this value, the more negligible its effect on estimates\footnote{We observed that this parameter did not play a crucial role in practice.}). 
We run a first set of experiments in order to set all other parameters (eligibility factor and learning rates). We use the following schedule for the learning rates:
\begin{equation}
  \alpha_i = \alpha_0  \frac{\alpha_c}{\alpha_c + i} \text{ and } \beta_i = \beta_0 \frac{\beta_c}{\beta_c + i^{\frac{2}{3}}}.
\end{equation}
More precisely, we generate one problem (MDP and policy) for each possible combination small/big on-policy/off-policy (leading to four problems). For each problem, we generate 10 trajectories of length $10^4$ using the behaviorial policy (which is the randomly generated target policy in the on-policy case and the uniform policy in the off-policy case), to be used by all algorithms. For each meta-parameter, we consider the following ranges of values: $\lambda \in \{0, 0.4, 0.7, 0.9, 1\}$, $\alpha_0 \in \{10^{-2}, 10^{-1}, 10^0\}$, $\alpha_c \in \{10^{1}, 10^{2}, 10^3\}$, $\beta_0 \in \{10^{-2}, 10^{-1}, 10^0\}$ and $\beta_c \in \{10^{1}, 10^{2}, 10^3\}$. Then, we compute the parameter estimates considering all algorithms instantiated with each possible combination of the meta-parameters. This gives for each combination a family $\theta_{i,d}$ with $i$ the number of transitions encountered in the $d^\text{th}$ trajectory. Finally, for each problem and each algorithm, we choose the combination of meta-parameters which minimizes the average error on the second half of the learning curves (we do this to reduce the sensitivity to the initialization and the transient behavior). Formally, we pick the set of parameters that minimizes the following quantity:
\begin{equation}
  \text{err} = \frac{1}{10}  \sum_{d=1}^{10} \frac{1}{5.10^3} \sum_{i=5.10^3}^{10^4} \|\Phi\theta_{i,d} - V^\pi\|_2.
\end{equation}
We provide the empirical results of this first set of experiments in Table~\ref{tab:small_on} to~\ref{tab:big_off}. As a complement, we detail in Figure~\ref{fig:comparisonlambda} the sensitivity of all algorithms with respect to the \emph{main} parameter $\lambda$ that controls the eligibility traces. We comment these results below.

\begin{table}[tbh]
  \begin{center}
  \begin{tabular}{|l||r|r|r|r|r||r|}
    \hline
     &  $\lambda$ & $\alpha_0$ & $\alpha_c$ & $\beta_0$ & $\beta_c$ & err
    \\
    \hline \hline
    LSTD & 0.9 & ~ & ~ & ~ & ~ & 1.60 \\
    \hline
    LSPE & 0.9 & ~ & ~ & ~ & ~ & 1.60 \\
    \hline
    FPKF & 1 & ~ & ~ & ~ & ~ & 1.60 \\
    \hline
    BRM & 0.9 & ~ & ~ & ~ & ~ & 1.60 \\
    \hline \hline
    TD & 0 & $10^{-1}$ & $10^3$ & ~ & ~ & 1.76 \\
    \hline
    gBRM & 0.7 & $10^{-1}$ & $10^2$ & ~ & ~ & 1.68 \\
    \hline
    TDC & 0.9 & $10^{-1}$ & $10^3$ & $10^{-1}$ & $10^3$ & 2.11 \\
    \hline
    GTD2 & 0.7 & $10^{-1}$ & $10^3$ & $10^{-1}$ & $10^3$ & 1.69 \\
    \hline
  \end{tabular}
  \caption{Small problem ($\mathcal{G}(30, 4, 2, 8)$), on-policy learning ($\pi = \pi_0$).}
  \label{tab:small_on}
  \end{center}
\end{table}
Table~\ref{tab:small_on} shows the best meta-parameters for 10
trajectories of a single instance of a small Garnet problem in an
on-policy setting, as well as related efficiencies. Numerically, all
least-squares-based methods provide equivalent performance, with
similar choices of the eligibility factor (which is the only
meta-parameter). TD gets its best results with a small $\lambda$ (we
believe this is the case because the MDP has few states). The new gBRM
algorithm and GTD2 algorithm work well, but picking a higher value of
$\lambda$. Eventually, TDC performs slightly
worse. Figure~\ref{fig:comparisonlambda} (top, left) shows that the
choice of $\lambda$ does not matter, except for FPKF that requires
$\lambda=1$ to be efficient; with this value FPKF is almost identical
to LSPE(1) and LSTD(1) (cf. the discussion at the end of
Section~\ref{sec:fpkf}).

\begin{table}[tbh]
  \begin{center}
  \begin{tabular}{|l||r|r|r|r|r||r|}
    \hline
     &  $\lambda$ & $\alpha_0$ & $\alpha_c$ & $\beta_0$ & $\beta_c$ & err
    \\
    \hline \hline
    LSTD & 0.4 & ~ & ~ & ~ & ~ & 2.75 \\
    \hline
    LSPE & 0.7 & ~ & ~ & ~ & ~ & 2.75 \\
    \hline
    FPKF & 1 & ~ & ~ & ~ & ~ & 2.77 \\
    \hline
    BRM & 0.9 & ~ & ~ & ~ & ~ & 2.77 \\
    \hline \hline
    TD & 0.4 & $10^{-1}$ & $10^3$ & ~ & ~ & 3.95 \\
    \hline
    gBRM & 0.9 & $10^{-1}$ & $10^3$ & ~ & ~ & 4.45 \\
    \hline
    TDC & 0.9 & $10^{-1}$ & $10^3$ & $10^{-1}$ & $10^3$ & 6.31 \\
    \hline
    GTD2 & 0.9 & $10^{-2}$ & $10^3$ & $10^{-1}$ & $10^3$ & 11.90 \\
    \hline
  \end{tabular}
  \caption{Big problem ($\mathcal{G}(100, 10, 3, 20)$), on-policy learning ($\pi = \pi_0$).}
  \label{tab:big_on}
  \end{center}
\end{table}
Table~\ref{tab:big_on} shows the best meta-parameters for 10
trajectories of a single instance of a big Garnet problem in an
on-policy setting, as well as related performance.  These results are
consistent with those of the small problem, in the on-policy setting
(with slightly different meta-parameters). The main difference are
that GTD2 performs here the worse. Figure~\ref{fig:comparisonlambda}
(top, right) suggests that as the problem's size grows, the role of the
eligibity factor gets more prominent: most algorithms need a
relatively high value of $\lambda$ to perform the best.

\begin{figure}[tbh]
\begin{minipage}[c]{.49\linewidth}
\begin{center}
\includegraphics[width=\linewidth]{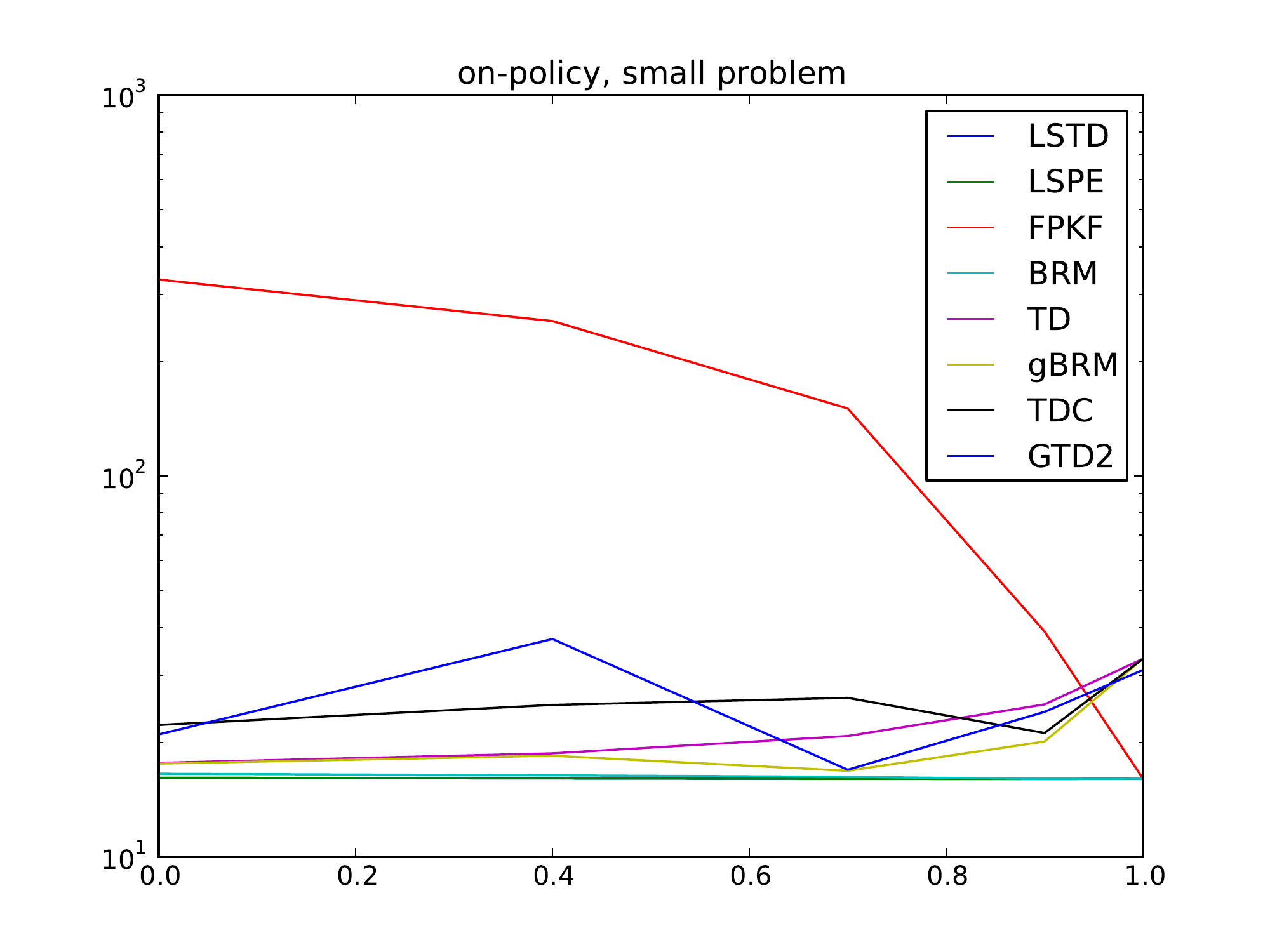}
\end{center}
\end{minipage}
\begin{minipage}[c]{.49\linewidth}
\begin{center}
\includegraphics[width=\linewidth]{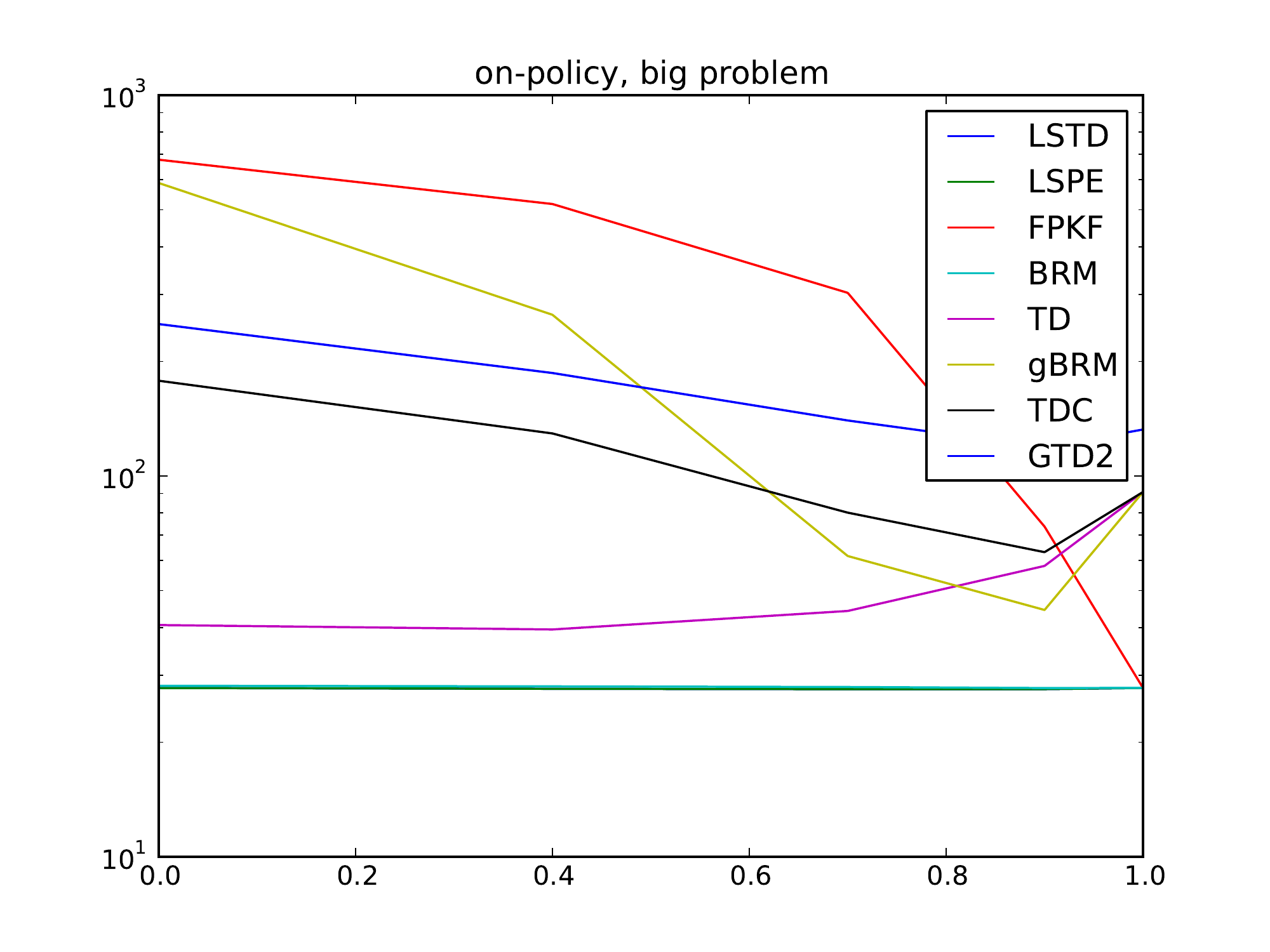}
\end{center}
\end{minipage}

\begin{minipage}[c]{.49\linewidth}
\begin{center}
\includegraphics[width=\linewidth]{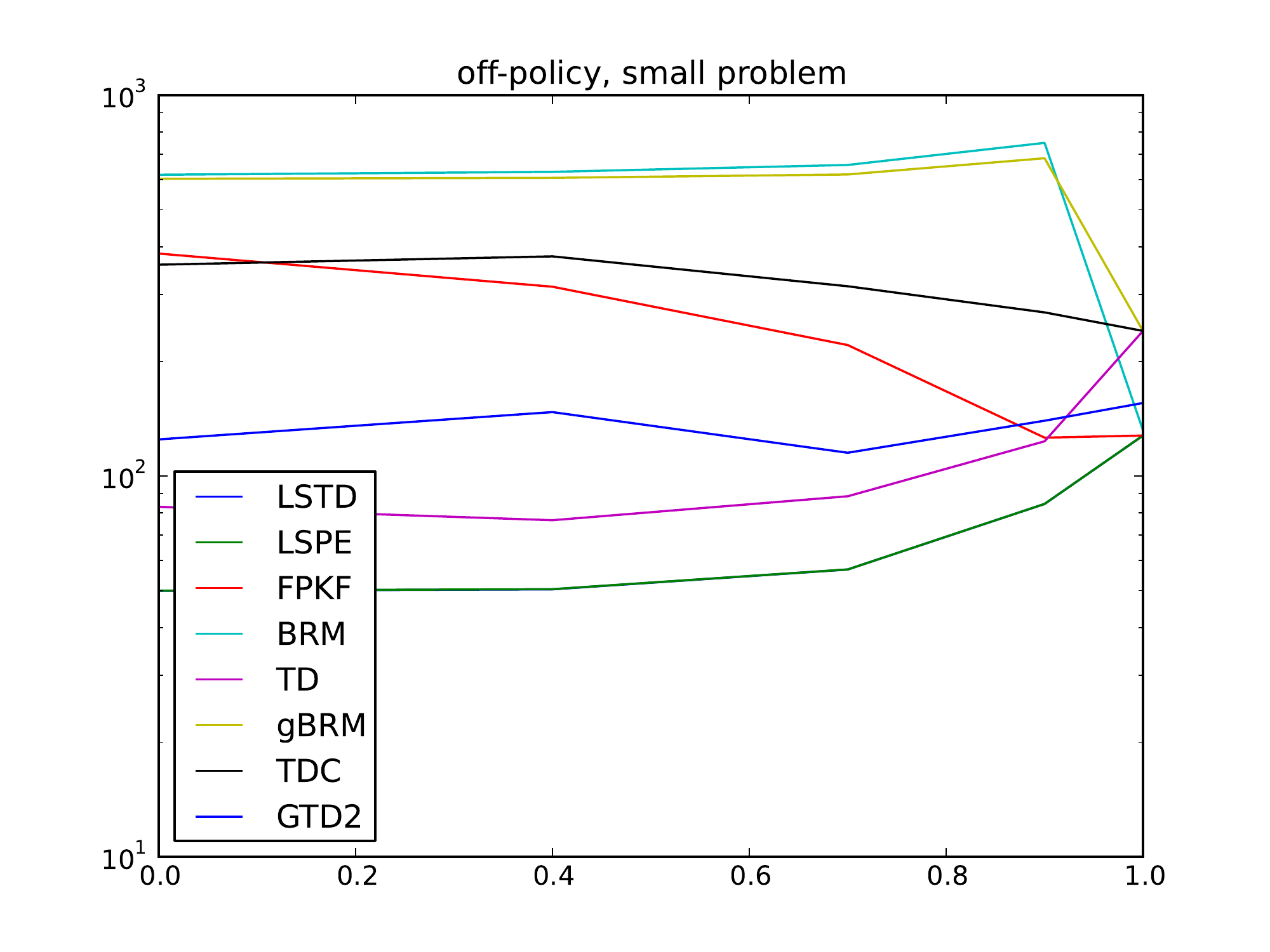}
\end{center}
\end{minipage}
\begin{minipage}[c]{.49\linewidth}
\begin{center}
\includegraphics[width=\linewidth]{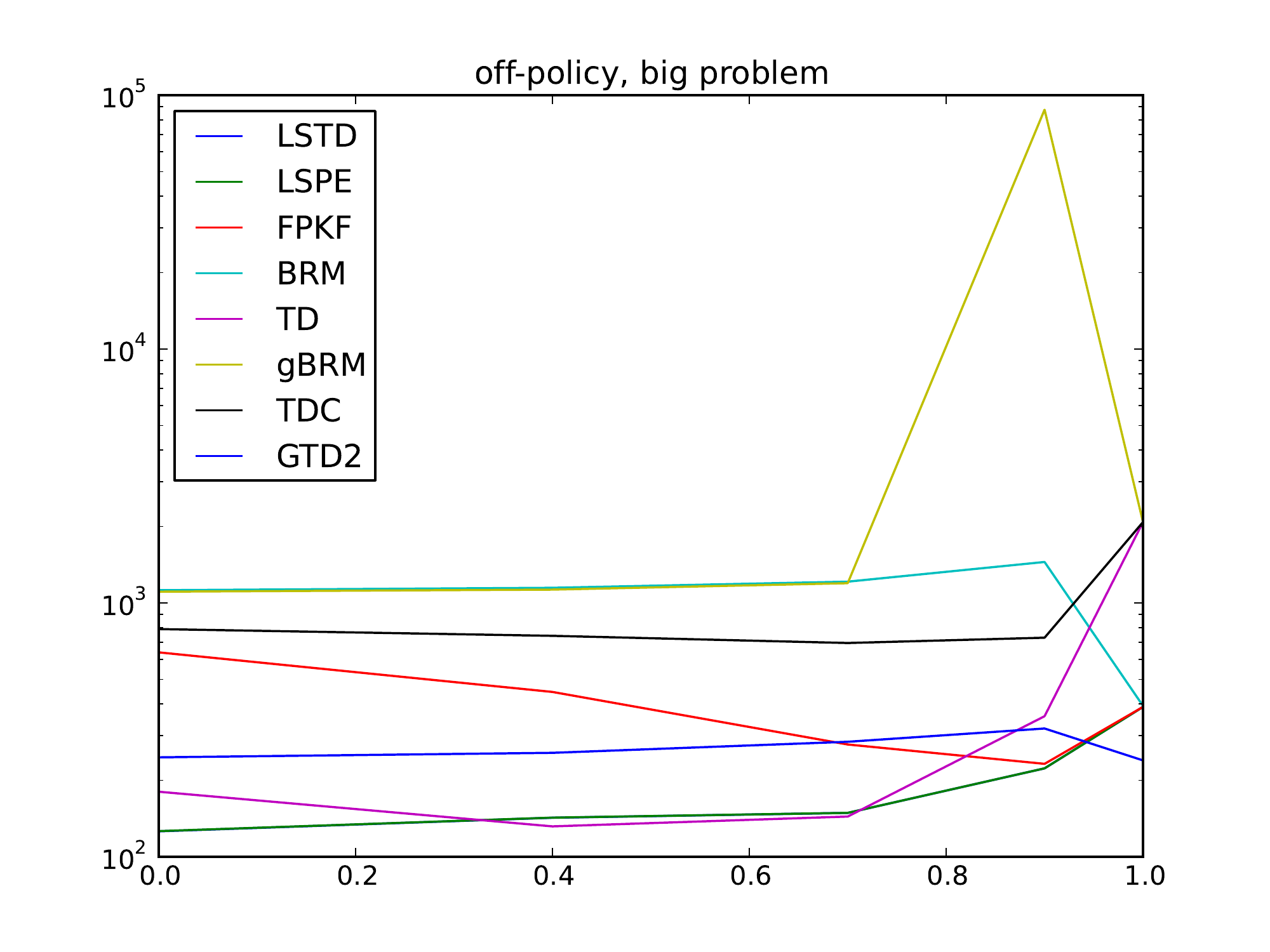}
\end{center}
\end{minipage}
\caption{Sensitivity of performance of the algorithms ($y$-axis, in logarithmic scale) with respect to the eligibility trace parameter $\lambda$ ($x$-axis). Left: Small problem ($\mathcal{G}(30, 4, 2, 8)$), right: Big problem ($\mathcal{G}(100, 10, 3, 20)$). Top: on-policy learning ($\pi = \pi_0$), bottom: off-policy learning ($\pi \neq \pi_0$).}
\label{fig:comparisonlambda}
\end{figure}

\begin{table}[tbh]
  \begin{center}
  \begin{tabular}{|l||r|r|r|r|r||r|}
    \hline
     &  $\lambda$ & $\alpha_0$ & $\alpha_c$ & $\beta_0$ & $\beta_c$ & err
    \\
    \hline \hline
    LSTD & 0 & ~ & ~ & ~ & ~ & 4.99 \\
    \hline
    LSPE & 0 & ~ & ~ & ~ & ~ & 5.00 \\
    \hline
    FPKF & 0.9 & ~ & ~ & ~ & ~ & 12.61 \\
    \hline
    BRM & 1 & ~ & ~ & ~ & ~ & 13.17 \\
    \hline \hline
    TD & 0.4 & $10^{-1}$ & $10^2$ & ~ & ~ & 7.65 \\
    \hline
    gBRM & 1 & $10^{-2}$ & $10^2$ & ~ & ~ & 24.04 \\
    \hline
    TDC & 1 & $10^{-2}$ & $10^2$ & $10^{-2}$ & $10^1$ &  24.04  \\
    \hline
    GTD2 & 0.7 & $10^{-1}$ & $10^3$ & $10^{-2}$ & $10^1$ & 11.51 \\
    \hline
  \end{tabular}
  \caption{Small problem ($\mathcal{G}(30, 4, 2, 8)$), off-policy learning ($\pi \neq \pi_0$).} 
  \label{tab:small_off}
  \end{center}
\end{table}
Table~\ref{tab:small_off} reports the best meta-parameters in an
off-policy setting for a small problem. In terms of performance,
Least-squares approaches are no longer equivalent. LSTD and LSPE get
the best results, with the smallest possible value $\lambda=0$: we
believe that this choice is due to the fact that higher eligibility
factor increases the variance due to importance sampling. FPKF and BRM
need large values of $\lambda$ to work well, and suffer much more from
the off-policy aspect. Figure~\ref{fig:comparisonlambda} (bottom,
left) suggests that FPKF and BRM need a high value of $\lambda$ (to
``catch'' the good performance of LSTD/LSPE) but then suffers from the
variance due to importance sampling. Regarding gradient-based methods,
TD's performance is good (it is close that of LSTD/LSPE), followed
closely by GTD2 (both being better than FPKF/BRM). TDC and gBRM lead
to the worse results; as both methods choose $\lambda=1$, they here
turn out to be equivalent to TD(1) (cf. the discussions after
Algorithms~\ref{algo:tdc} and \ref{algo:gbrm})\footnote{In particular,
one can observe that the performance of gBRM(1) and TDC(1) in
Table~\ref{tab:small_off} are numerically equal.}.  As for FPKF/BRM
with respect to LSTD/LSPE, Figure~\ref{fig:comparisonlambda} (bottom,
left) further suggests that TDC and gBRM need a high value of
$\lambda$ in order to get a reasonable performance, but then suffer
from the variance of importance sampling.
\begin{table}[tbh]
  \begin{center}
  \begin{tabular}{|l||r|r|r|r|r||r|}
    \hline
     &  $\lambda$ & $\alpha_0$ & $\alpha_c$ & $\beta_0$ & $\beta_c$ & err
    \\
    \hline \hline
    LSTD & 0 & ~ & ~ & ~ & ~ & 12.60 \\
    \hline
    LSPE & 0 & ~ & ~ & ~ & ~ & 12.62 \\
    \hline
    FPKF & 0.9 & ~ & ~ & ~ & ~ & 23.24 \\
    \hline
    BRM & 1 & ~ & ~ & ~ & ~ & 39.32 \\
    \hline \hline
    TD & 0.4 & $10^{-1}$ & $10^1$ & ~ & ~ & 13.18 \\
    \hline
    gBRM & 0 & $10^{-2}$ & $10^1$ & ~ & ~ & 110.74 \\
    \hline
    TDC & 0.7 & $10^{-2}$ & $10^3$ & $10^{-2}$ & $10^1$ & 69.53 \\
    \hline
    GTD2 & 1 & $10^{-2}$ & $10^1$ & $10^{-1}$ & $10^3$ & 23.97 \\
    \hline
  \end{tabular}
  \caption{Big problem ($\mathcal{G}(100, 10, 3, 20)$), off-policy learning ($\pi \neq \pi_0$).}
  \label{tab:big_off}
  \end{center}
\end{table}
Eventually, Table~\ref{tab:big_off} shows the meta-parameters and performance in the most difficult situation:  the off-policy setting of the big problem. These results are consistent with the off-policy results of the small problem, summarized in Table~\ref{tab:small_off}.  LSTD and LSPE are the most efficient algorithms and choose the smallest possible value $\lambda=0$. FPKF and BRM's performance deteriorate (significantly for the latter). TD behaves reasonably good (it is in particular much better than FPKF) and GTD2 follows closely. The performance of TDC and gBRM are the worse, the latter's by a significant amount. Figure~\ref{fig:comparisonlambda} (bottom, right) is similar to that of the small problem.


The main goal of the series of experiments we have just described was to choose reasonable values for the meta-parameters. We have also used these experiments to quickly comment the relative performance of the algorithms, but this is not statistically significant as this was based on a single (random) problem. Though we will see that the general behavior of the algorithm is globally consistent with what we have seen so far, the  series of experiments that we are about to describe  aims at providing such a statistically significant performance comparison. For each situation (small and big problems, on- and off-policy), we fix the meta-parameters to the previously reported values and we compare the algorithms on several new instances of the problems. These results are reported on Figures.~\ref{fig:small_on} to~\ref{fig:big_off}. For each of the 4 problems, we randomly generate 100 instances (MDP and policy to be evaluated). For each such problem, we generate a trajectory of length $10^5$. Then, all algorithms learn using this very trajectory. On each figure, we report the average performance (left), measured as the difference between the true value function (computed from the model) and the currently estimated one, $\|V^\pi - \Phi \theta\|_2$, as well as the associated standard deviation (right).

\begin{figure}[tbh]
\begin{minipage}[c]{.49\linewidth}
\begin{center}
\includegraphics[width=\linewidth]{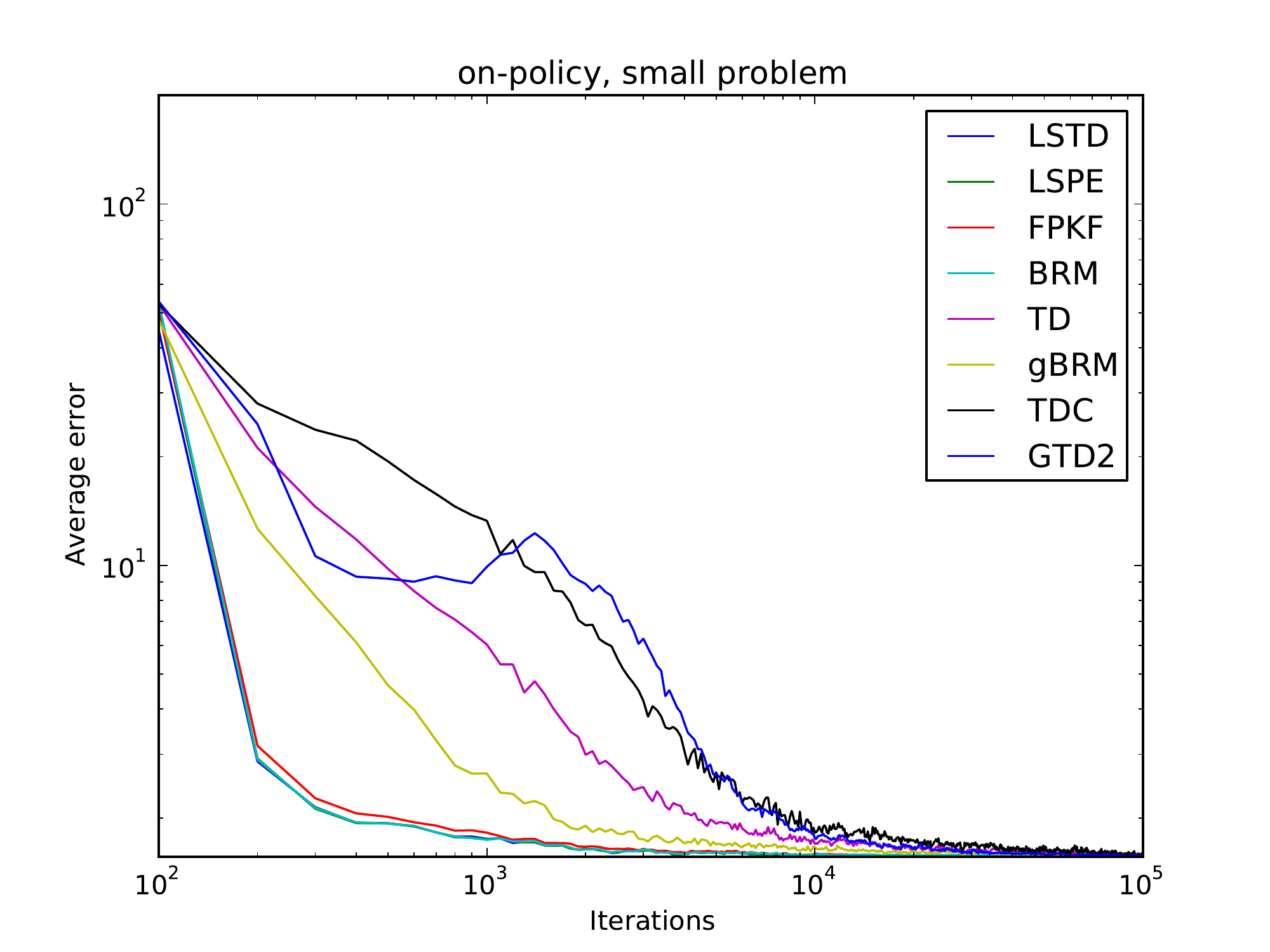}
\end{center}
\end{minipage}
\begin{minipage}[c]{.49\linewidth}
\begin{center}
\includegraphics[width=\linewidth]{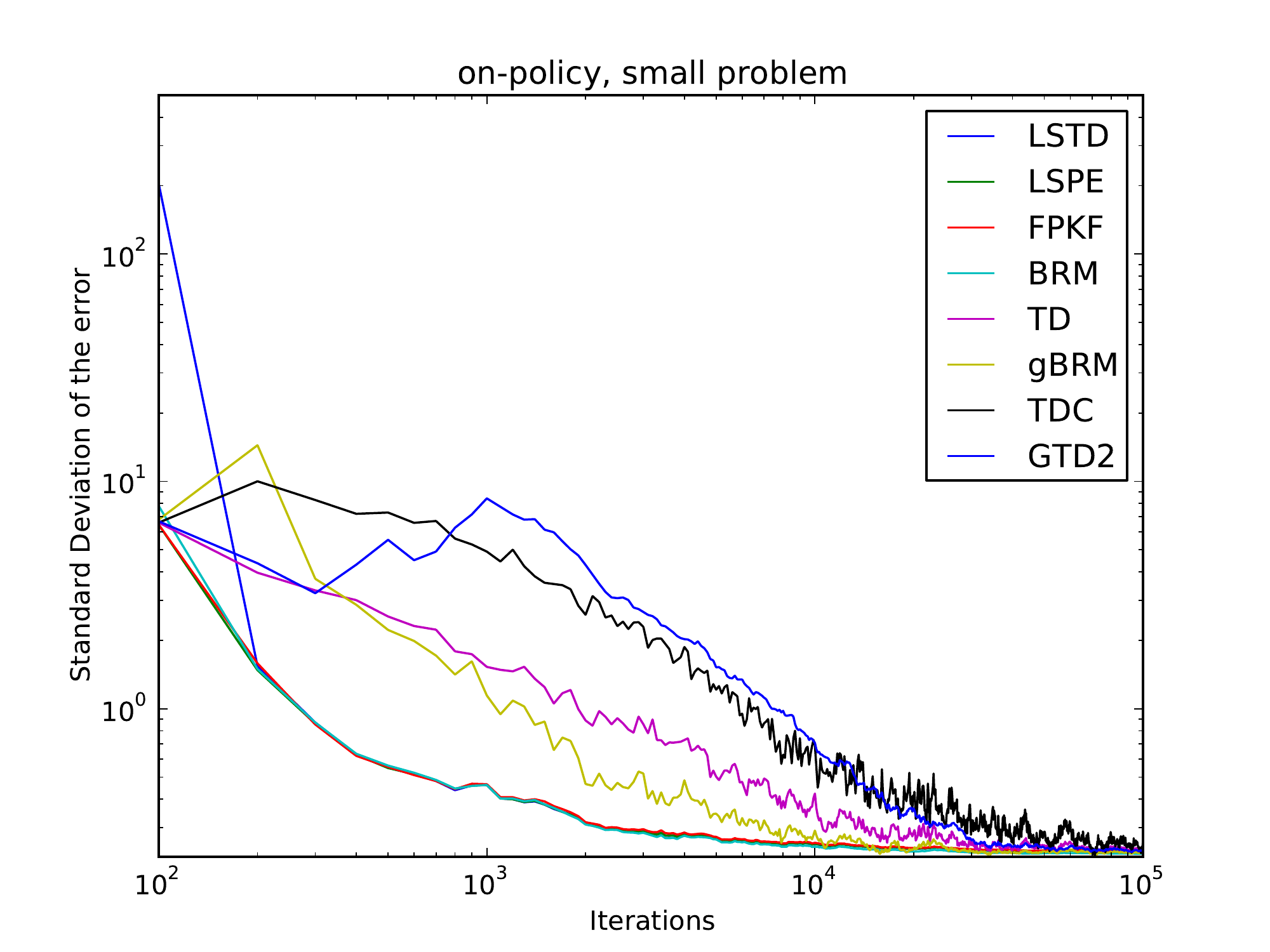}
\end{center}
\end{minipage}
\caption{Performance for small problems ($\mathcal{G}(30, 4, 2, 8)$), on-policy learning ($\pi = \pi_0$) (left: average error, right: standard deviation).}
\label{fig:small_on}
\end{figure}
\begin{figure}[tbh]
\begin{minipage}[c]{.49\linewidth}
\begin{center}
\includegraphics[width=\linewidth]{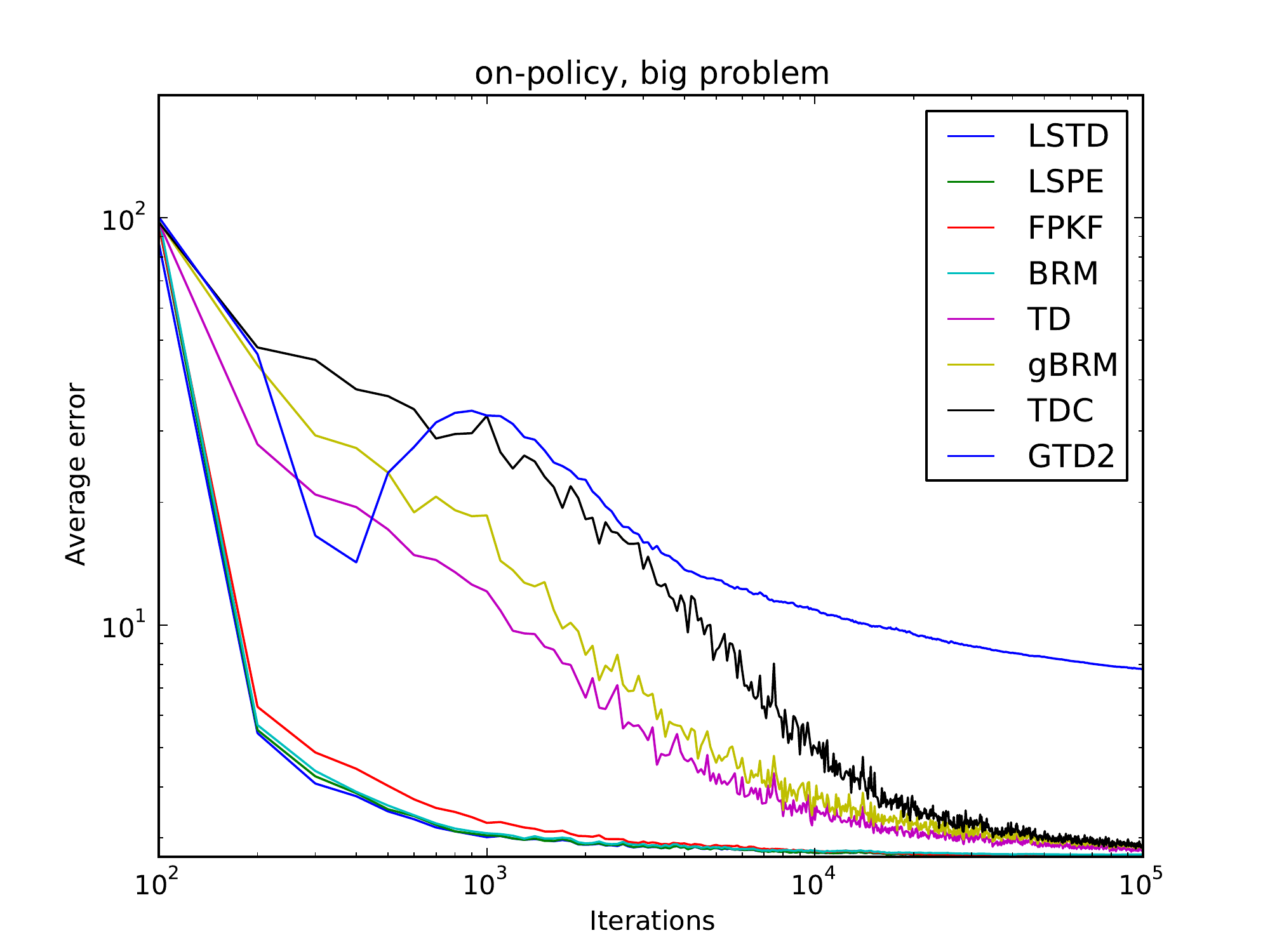}
\end{center}
\end{minipage}
\begin{minipage}[c]{.49\linewidth}
\begin{center}
\includegraphics[width=\linewidth]{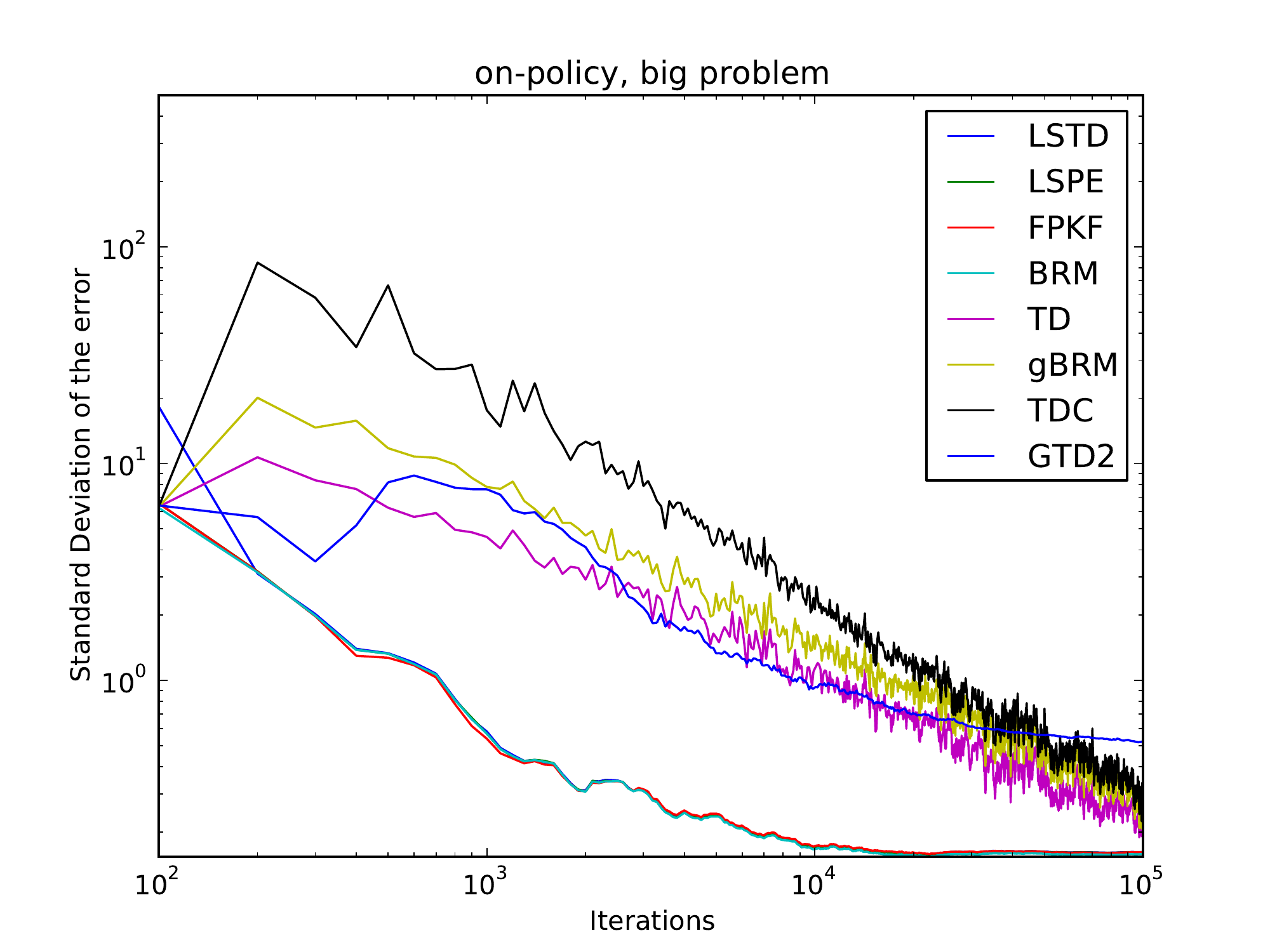}
\end{center}
\end{minipage}
\caption{Performance for big problems ($\mathcal{G}(100, 10, 3, 20)$), on-policy learning ($\pi = \pi_0$) (left: average error, right: standard deviation).}
\label{fig:big_on}
\end{figure}
We begin by discussing the results in the on-policy setting. 
Figure~\ref{fig:small_on} compares all algorithms  for 100 randomly generated small problems (that is, each run corresponds to different dynamics, reward function, features and evaluated policy), the meta-parameters being those provided in Table~\ref{tab:small_on}. All least-squares approaches provide the best results and are bunched together; this was to be expected, as all algorithms use $\lambda$ close to $1$. In these problems, gBRM works better than other gradient-based methods, followed by TD and GTD2/TDC. 
%
Figure~\ref{fig:big_on} compares the algorithms for 100 randomly generated big problems, the meta-parameters being those provided in Table~\ref{tab:big_on}. These result are similar to those of the small problem in an off-policy setting, except that TDC is now faster than GTD2,  that TD is slightly faster than gBRM. 

\begin{figure}[tbh]
\begin{minipage}[c]{.49\linewidth}
\begin{center}
\includegraphics[width=\linewidth]{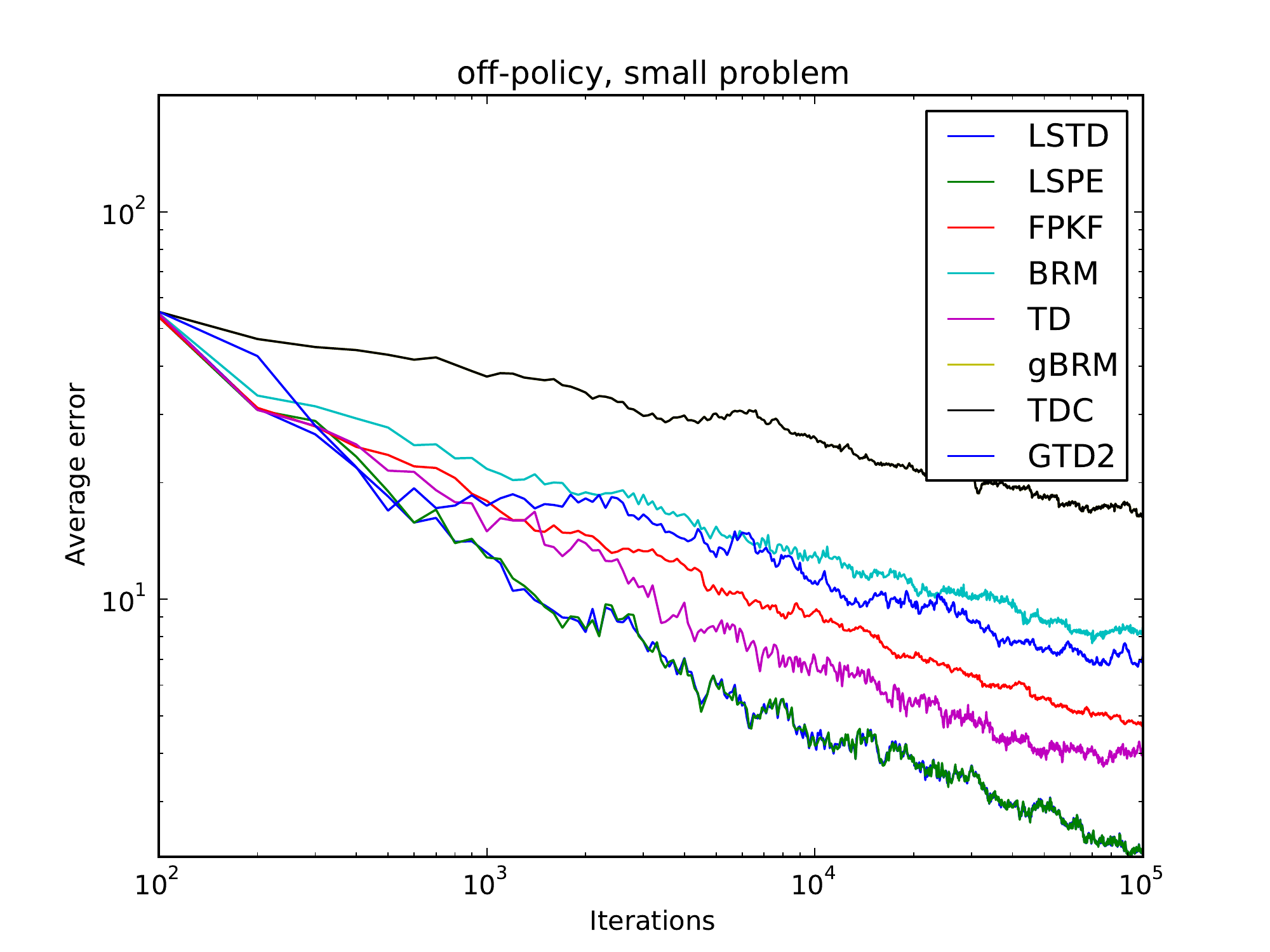}
\end{center}
\end{minipage}
\begin{minipage}[c]{.49\linewidth}
\begin{center}
\includegraphics[width=\linewidth]{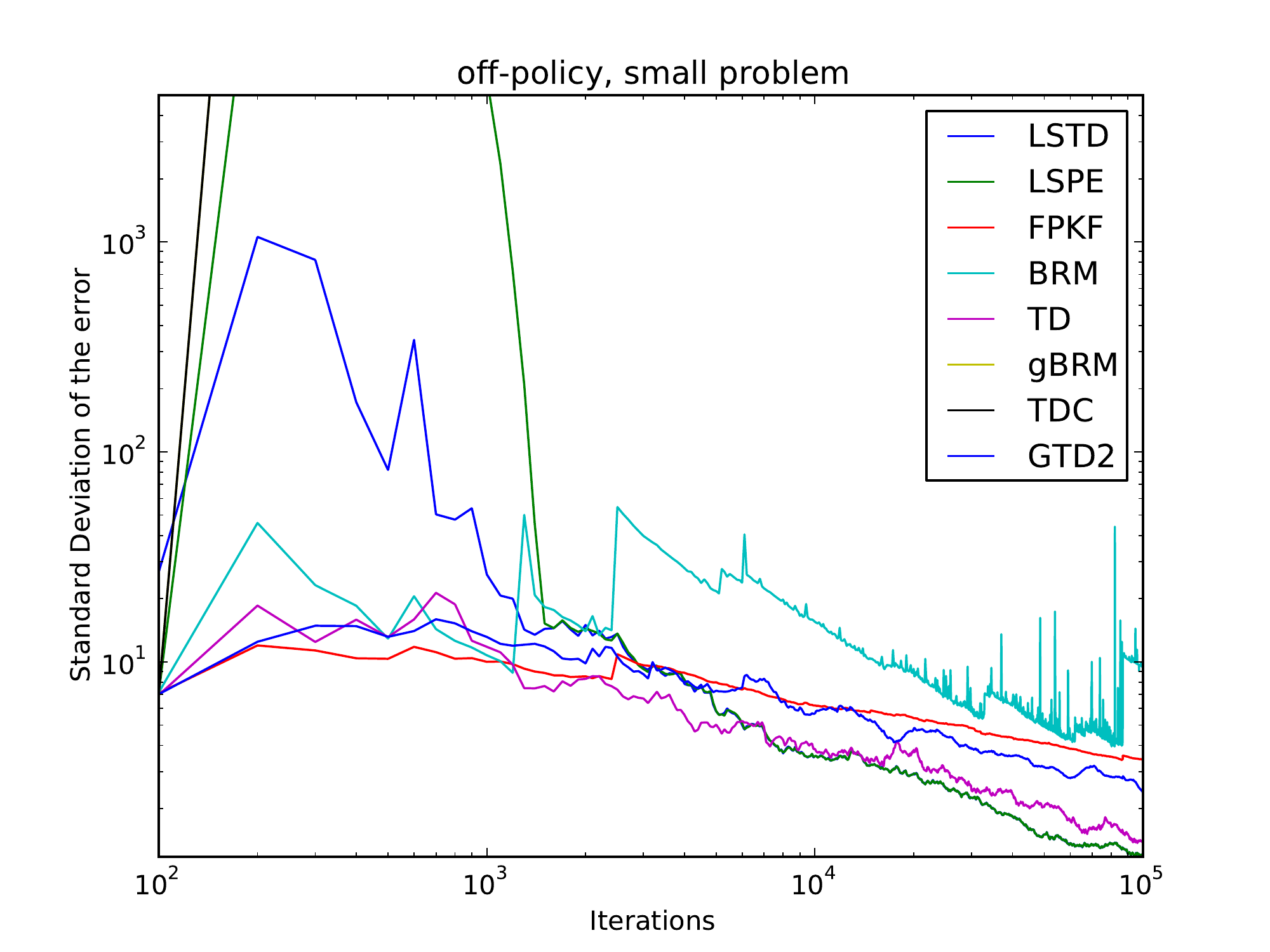}
\end{center}
\end{minipage}
\caption{Performance for small problems ($\mathcal{G}(30, 4, 2, 8)$), off-policy learning ($\pi \neq \pi_0$) (left: average error, right: standard deviation). }
\label{fig:small_off}
\end{figure}
\begin{figure}[tbh]
\begin{minipage}[c]{.49\linewidth}
\begin{center}
\includegraphics[width=\linewidth]{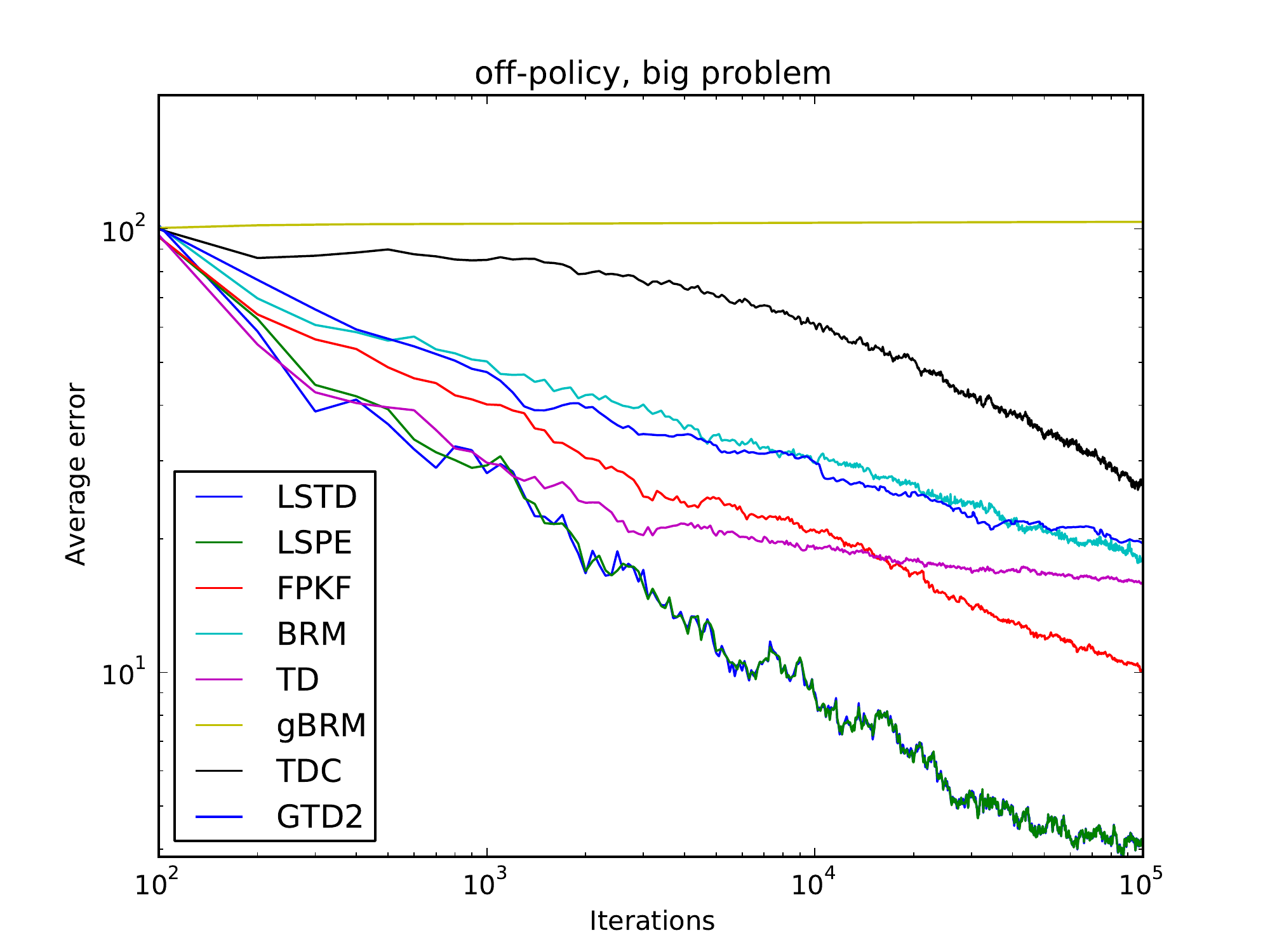}
\end{center}
\end{minipage}
\begin{minipage}[c]{.49\linewidth}
\begin{center}
\includegraphics[width=\linewidth]{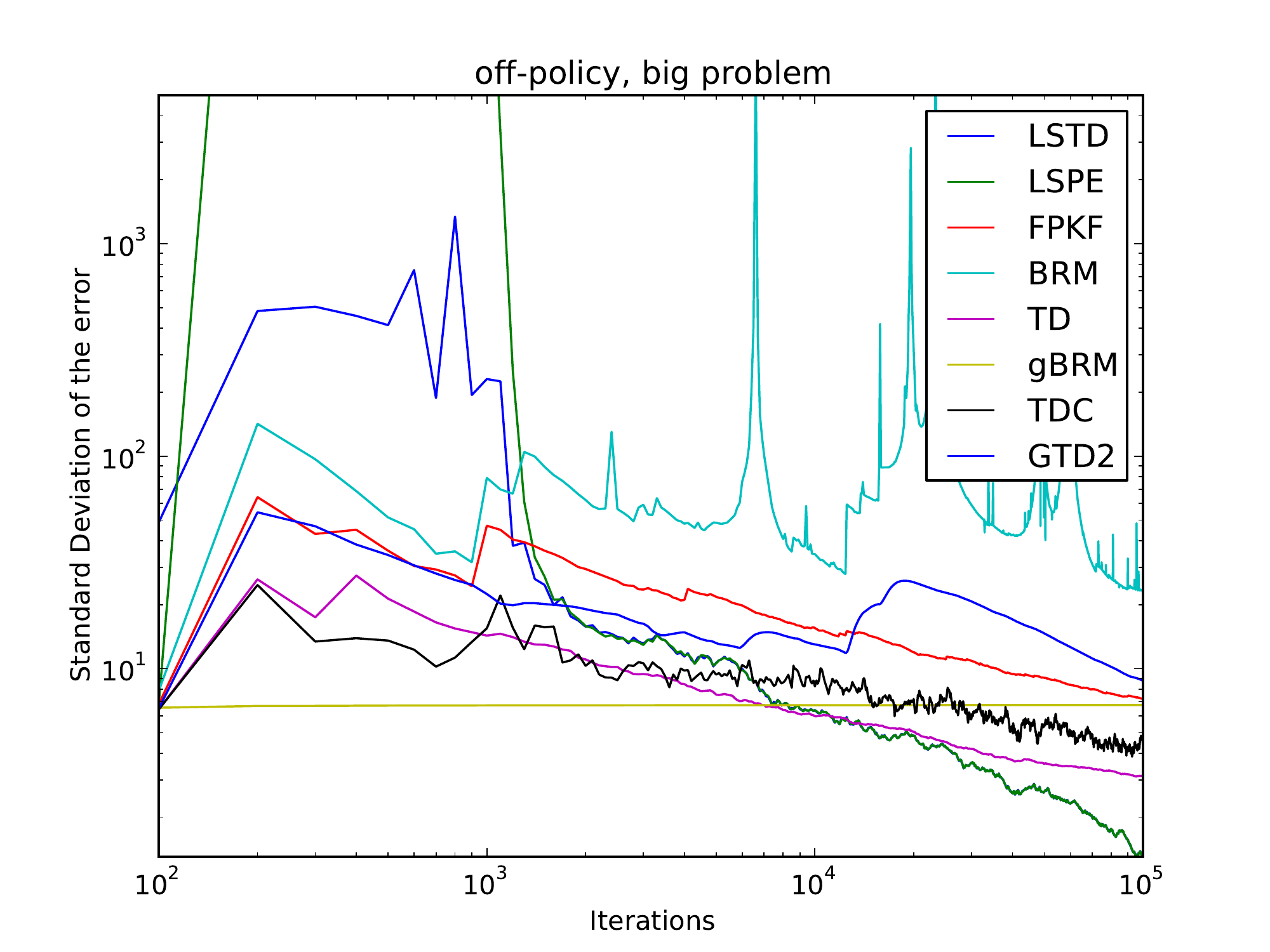}
\end{center}
\end{minipage}
\caption{Performance for big problems ($\mathcal{G}(100, 10, 3, 20)$), on-policy learning ($\pi \neq \pi_0$) (left: average error, right: standard deviation).}
\label{fig:big_off}
\end{figure}
We now consider the off-policy setting. Figure~\ref{fig:small_off}
provides the average performance and standard deviation of the
algorithms (meta-parameters being those of Table~\ref{tab:small_off})
on 100 small problems. Once again, we can see that LSTD/LSPE provide
the best results. The two other least-squares methods (FPKF and BRM)
are overtaken by the gradient-based TD algorithm. GTD2 is quite slow
(slower than TD) and TDC/gBRM (that are identical to TD(1) since they
both use $\lambda=1$) are the slowest algorithms.
Figure~\ref{fig:big_off} provides the same data for the big problems
(with the meta-parameters of Table~\ref{tab:big_off}). These results
are similar to those of the small problems in an off-policy
setting. The main differences are 1) that FPKF appears to converge
faster than TD but with a bigger standard deviation, and 2) gBRM (that
uses $\lambda=0$) does not work anymore (it is here equivalent to the
standard algorithm by \citep{Baird:1995} and probably suffers from the
well-known associated bias issue).

\paragraph{Summary}
Overall, our experiments suggest that the two best algorithms are
LSTD/LSPE, since they converge much faster in all situations. The
gradient-based TD algorithm globally diplays a good behavior and
constitutes a good alternative when the number $p$ of features is too
big for least-squares methods to be implemented.  Though some new
algorithms/extensions show interesting results (FPKF($\lambda$) is
consistently better that the state-of-the-art FPKF
by \cite{Choi:2006}, gBRM works well in the on-policy setting) most of
the other algorithms do not seem to be empirically competitive with
the trio LSTD/LSPE/TD, especially in off-policy situations. In
particular, the algorithm introduced specifically for the off-policy
setting (TDC/GTD2) are much slower than TD. Moreover, the condition
required for the good behavior of LSPE, FPKF and TD -- the contraction
of $\Pi_0 T^\lambda$ -- does not seem to be very restrictive in
practice (at least for the Garnet problems we considered): though it
is possible to build specific pathological examples where these
algorithms diverge\footnote{A preliminary version of this
article~\citep{scherrer:2011} contains such examples, and also an
example where an adverserial choice of $\lambda$ leads to the
divergence of LSTD($\lambda$).}, this never happened in our
experiments.

\section{Conclusion and Future Work}
\label{sec:conclusion}

We have considered least-squares and gradient-based algorithms for
value estimation in an MDP context.  Starting from the on-policy case
with no trace, we have recalled that several algorithms (LSTD, LSPE,
FPKF and BRM for least-squares approaches, TD, gBRM and
TDC/GTD2 for gradient-based approaches) fall in a common algorithmic
pattern (\eqrefb{unify1}). Substituting the original Bellman operator
by an operator that deals with traces and off-policy samples naturally
leads to the state-of-the-art off-policy trace-based versions of LSTD,
LSPE, TD and TDC, and suggests natural extensions of FPKF, BRM, gBRM
and GTD2. This way, we surveyed many known and new off-policy
eligibility traces-based algorithms for policy evaluation. 

We have explained how to derive recursive (memory and time-efficient)
implementations of all these algorithms and discussed their known
convergence properties (including an original analysis of
BRM($\lambda$) for sufficiently small $\lambda$, that implies the so
far not known convergence of GPTD/KTD). Interestingly, it appears that
the analysis of off-policy traces-based stochastic gradient algorithms
under mild assumptions is still an open problem: the only currently
known analysis of TD~\citep{Yu:2010} only applies to a constrained
version of the algorithm, and that of TDC~\citep{Maei:2010} relies on
an assumption on the boundedness of the second moment traces that is
restrictive~\citep{Yu:2010}. Filling this theoretical gap, as well as
providing complete analyses for the other gradient algorithms and
FPFK($\lambda$) and BRM($\lambda$) constitute important future work.

Finally, we have illustrated and compared the behavior of these
algorithms; this constitutes the first exhaustive empirical
comparison of linear methods\footnote{To our knowledge, there does not
even exist any work reporting and comparing empirical results of
LSTD(0) and FPKF(0).}.  Overall, our study suggests that even if the
use of eligibility traces generally improves the efficiency of all
algorithms, LSTD and LSPE consistently provide the best estimates; and
in situations where the computational cost is prohibitive for a
least-squares approach (when the number $p$ of features is large), TD
probably constitutes the best alternative.


%
%

\bibliography{biblio}

\begin{thebibliography}{}

\bibitem[Antos {\em et~al.}(2006)Antos, Szepesv{\'a}ri, and Munos]{antos2006}
Antos, A., Szepesv{\'a}ri, C., and Munos, R. (2006).
\newblock {Learning Near-optimal Policies with {B}ellman-residual Minimization
  based Fitted Policy Iteration and a Single Sample Path}.
\newblock In {\em COLT\/}.

\bibitem[Archibald {\em et~al.}(1995)Archibald, McKinnon, and
  Thomas]{Archibald:95}
Archibald, T., McKinnon, K., and Thomas, L. (1995).
\newblock {On the Generation of Markov Decision Processes}.
\newblock {\em Journal of the Operational Research Society\/}, {\bf 46},
  354--361.

\bibitem[Baird(1995)Baird]{Baird:1995}
Baird, L.~C. (1995).
\newblock {Residual Algorithms: Reinforcement Learning with Function
  Approximation}.
\newblock In {\em ICML\/}.

\bibitem[Bertsekas and Ioffe(1996)Bertsekas and Ioffe]{ioffe}
Bertsekas, D. and Ioffe, S. (1996).
\newblock Temporal differences-based policy iteration and applications i n
  neuro-dynamic programming.
\newblock Technical report, MIT.

\bibitem[Bertsekas and Tsitsiklis(1996)Bertsekas and
  Tsitsiklis]{Bertsekas:1996}
Bertsekas, D.~P. and Tsitsiklis, J.~N. (1996).
\newblock {\em {Neuro-Dynamic Programming}\/}.
\newblock {Athena Scientific}.

\bibitem[Bertsekas and Yu(2009a)Bertsekas and Yu]{bertsekas:09proj}
Bertsekas, D.~P. and Yu, H. (2009a).
\newblock Projected equation methods for approximate solution of large linear
  systems.
\newblock {\em Journal of Computational and Applied Mathematics\/}, {\bf 227},
  27--50.

\bibitem[Bertsekas and Yu(2009b)Bertsekas and Yu]{Bertsekas:2009}
Bertsekas, D.~P. and Yu, H. (2009b).
\newblock {Projected Equation Methods for Approximate Solution of Large Linear
  Systems}.
\newblock {\em J. Comp. and Applied Mathematics\/}, {\bf 227}(1), 27--50.

\bibitem[Bottou and Bousquet(2011)Bottou and Bousquet]{bottou-bousquet-2011}
Bottou, L. and Bousquet, O. (2011).
\newblock The tradeoffs of large scale learning.
\newblock In S.~Sra, S.~Nowozin, and S.~J. Wright, editors, {\em Optimization
  for Machine Learning\/}, pages 351--368. MIT Press.

\bibitem[Bradtke and Barto(1996)Bradtke and Barto]{Bradtke:1996}
Bradtke, S.~J. and Barto, A.~G. (1996).
\newblock {Linear Least-Squares algorithms for temporal difference learning}.
\newblock {\em Machine Learning\/}, {\bf 22}(1-3), 33--57.

\bibitem[Choi and Van~Roy(2006)Choi and Van~Roy]{Choi:2006}
Choi, D. and Van~Roy, B. (2006).
\newblock {A Generalized Kalman Filter for Fixed Point Approximation and
  Efficient Temporal-Difference Learning}.
\newblock {\em DEDS\/}, {\bf 16}, 207--239.

\bibitem[Engel(2005)Engel]{Engel:2005}
Engel, Y. (2005).
\newblock {\em {Algorithms and Representations for Reinforcement Learning}\/}.
\newblock Ph.D. thesis, Hebrew University.

\bibitem[Geist and Pietquin(2010a)Geist and Pietquin]{Supelec637}
Geist, M. and Pietquin, O. (2010a).
\newblock {Eligibility Traces through Colored Noises}.
\newblock In {\em {ICUMT}\/}.

\bibitem[Geist and Pietquin(2010b)Geist and Pietquin]{geist:jair}
Geist, M. and Pietquin, O. (2010b).
\newblock {Kalman Temporal Differences}.
\newblock {\em JAIR\/}, {\bf 39}, 483--532.

\bibitem[Geist and Pietquin(2013)Geist and Pietquin]{geist:vfa}
Geist, M. and Pietquin, O. (2013).
\newblock {An Algorithmic Survey of Parametric Value Function Approximation}.
\newblock {\em IEEE Transactions on Neural Networks and Learning Systems\/}.

\bibitem[Kearns and Singh(2000)Kearns and Singh]{kearns:2000}
Kearns, M. and Singh, S. (2000).
\newblock {Bias-Variance Error Bounds for Temporal Difference Updates}.
\newblock In {\em COLT\/}.

\bibitem[Kolter(2011)Kolter]{Kolter:2011}
Kolter, J.~Z. (2011).
\newblock {The Fixed Ponts of Off-Policy TD}.
\newblock In {\em Neural Information Processing Systems (NIPS)\/}.

\bibitem[Maei and Sutton(2010)Maei and Sutton]{Maei:2010}
Maei, H.~R. and Sutton, R.~S. (2010).
\newblock {GQ($\lambda$): A general gradient algorithm for temporal-difference
  prediction learning with eligibility traces}.
\newblock In {\em Conference on Artificial General Intelligence\/}.

\bibitem[Munos(2003)Munos]{Munos:2003}
Munos, R. (2003).
\newblock {Error Bounds for Approximate Policy Iteration}.
\newblock In {\em ICML\/}.

\bibitem[Nedi\'c and Bertsekas(2003)Nedi\'c and Bertsekas]{Nedic:2003}
Nedi\'c, A. and Bertsekas, D.~P. (2003).
\newblock {Least Squares Policy Evaluation Algorithms with Linear Function
  Approximation}.
\newblock {\em DEDS\/}, {\bf 13}, 79--110.

\bibitem[Precup {\em et~al.}(2000)Precup, Sutton, and Singh]{Precup:2000}
Precup, D., Sutton, R.~S., and Singh, S.~P. (2000).
\newblock {Eligibility Traces for Off-Policy Policy Evaluation}.
\newblock In {\em ICML\/}.

\bibitem[Precup {\em et~al.}(2001)Precup, Sutton, and Dasgupta]{Precup:2001}
Precup, D., Sutton, R.~S., and Dasgupta, S. (2001).
\newblock {Off-Policy Temporal-Difference Learning with Function
  Approximation}.
\newblock In {\em Proceedings of the 18th International Conference on Machine
  Learning\/}.

\bibitem[Randhawa and Juneja(2004)Randhawa and Juneja]{Randhawa}
Randhawa, R.~S. and Juneja, S. (2004).
\newblock {Combining importance sampling and temporal difference control
  variates to simulate Markov chains}.
\newblock {\em ACM Trans. Modeling and Computer Simulation\/}, {\bf 14}(1),
  1--30.

\bibitem[Ripley(1987)Ripley]{Ripley:87}
Ripley, B.~D. (1987).
\newblock {\em {Stochastic Simulation}\/}.
\newblock Wiley \& Sons.

\bibitem[Scherrer(2010)Scherrer]{scherrer:2010}
Scherrer, B. (2010).
\newblock {Should one compute the Temporal Difference fix point or minimize the
  Bellman Residual? The unified oblique projection view}.
\newblock In {\em ICML\/}.

\bibitem[Scherrer and Geist(2011)Scherrer and Geist]{scherrer:2011}
Scherrer, B. and Geist, M. (2011).
\newblock {Recursive Least-Squares Learning with Eligibility Traces}.
\newblock In {\em {European Wrokshop on Reinforcement Learning (EWRL 11)}\/},
  Athens, Greece.

\bibitem[Schoknecht(2002)Schoknecht]{Schoknecht:2002}
Schoknecht, R. (2002).
\newblock {Optimality of Reinforcement Learning Algorithms with Linear Function
  Approximation}.
\newblock In {\em NIPS\/}.

\bibitem[Sutton and Barto(1998)Sutton and Barto]{Sutton:1998}
Sutton, R.~S. and Barto, A.~G. (1998).
\newblock {\em {Reinforcement Learning: An Introduction (Adaptive Computation
  and Machine Learning)}\/}.
\newblock MIT Press, 3rd edition.

\bibitem[Sutton {\em et~al.}(2009)Sutton, Maei, Precup, Bhatnagar, Silver,
  Szepesv\'{a}ri, and Wiewiora]{Sutton:2009a}
Sutton, R.~S., Maei, H.~R., Precup, D., Bhatnagar, S., Silver, D.,
  Szepesv\'{a}ri, C., and Wiewiora, E. (2009).
\newblock Fast gradient-descent methods for temporal-difference learning with
  linear function approximation.
\newblock In {\em ICML '09: Proceedings of the 26th Annual International
  Conference on Machine Learning\/}.

\bibitem[Tsitsiklis and Van~Roy(1997)Tsitsiklis and Van~Roy]{vanroy:1997}
Tsitsiklis, J. and Van~Roy, B. (1997).
\newblock {An analysis of temporal-difference learning with function
  approximation}.
\newblock {\em IEEE Transactions on Automatic Control\/}, {\bf 42}(5),
  674--690.

\bibitem[Yu(2010a)Yu]{Yu:2010}
Yu, H. (2010a).
\newblock {Convergence of Least-Squares Temporal Difference Methods under
  General Conditions}.
\newblock In {\em ICML\/}.

\bibitem[Yu(2010b)Yu]{yu:2010tech}
Yu, H. (2010b).
\newblock {Least Squares Temporal Difference Methods: An Analysis Under General
  Condtions}.
\newblock Technical Report C-2010-39, University of Helsinki.

\bibitem[Yu and Bertsekas(2008)Yu and Bertsekas]{yu:2008}
Yu, H. and Bertsekas, D. (2008).
\newblock {New Error Bounds for Approximations from Projected Linear
  Equations}.
\newblock Technical Report C-2008-43, Dept. Computer Science, Univ. of
  Helsinki.

\end{thebibliography}


\newpage

\appendix

\section{Derivation of the recursive formulae for BRM($\lambda$)}
\label{brm:derivation}

We here detail the derivation of off-policy BRM($\lambda$). We will need two technical lemmata.
The first one is the Woodbury matrix identity which generalizes the
Sherman-Morrison formula (given in Lemma~\ref{lemma:sm}).
\begin{lemma}[Woodbury]\label{lemma:wmi}
  Let $A$, $U$, $C$ and $V$ be matrices of correct sizes, then:
  \begin{equation}
    (A + U C V)^{-1} = A^{-1} - A^{-1} U(C^{-1} + V A^{-1} U)^{-1} V
    A^{-1}
  \end{equation}
\end{lemma}
The second lemma is a rewriting of imbricated sums:
\begin{lemma}\label{lemma:sum2}
  Let $f\in\mathbb{R}^{\mathbb{N}\times \mathbb{N}\times \mathbb{N}}$ and $n\in\mathbb{N}$. We have:
  \begin{equation}
    \sum_{i=1}^n \sum_{j=i}^n \sum_{k=i}^n f(i,j,k) =
    \sum_{i=1}^n \sum_{j=1}^i \sum_{k=1}^j f(k,i,j) +
    \sum_{i=2}^n \sum_{j=1}^{i-1} \sum_{k=1}^j f(k,j,i)
  \end{equation}
\end{lemma}

As stated in \eqrefb{BRM:batchest}, we have the following batch estimate for BRM($\lambda$):
\begin{equation}
  \theta_i = \argmin_{\omega\in\mathbb{R}^p} \sum_{j=1}^i (z_{j\rightarrow  i} - \psi_{j\rightarrow i}^T\omega)^2
= (\tilde A_i)^{-1}\tilde b_i
\end{equation}
where
\begin{equation}
    \psi_{j\rightarrow i} = \sum_{k=j}^i \tilde{\rho}_j^{k-1} \Delta\phi_k
    \text{~and~}  z_{j\rightarrow i} = \sum_{k=j}^i \tilde{\rho}_j^{k-1} \rho_k r_k
\end{equation}
and
\begin{equation}
\label{eq:defAbBRM}
  \tilde A_i = \sum_{j=1}^i \psi_{j\rightarrow
  i} \psi_{j\rightarrow i}^T
\mbox{~~~and~~~}\tilde b_i =  \sum_{j=1}^i  \psi_{j\rightarrow  i} z_{j\rightarrow i}.
\end{equation}

To obtain a recursive formula, these two sums have to be reworked through Lemma~\ref{lemma:sum2}. Let us first focus on the latter:
\begin{align}
  \sum_{j=1}^i  \psi_{j\rightarrow i} z_{j\rightarrow i}
  &= \sum_{j=1}^i \sum_{k=j}^i \sum_{m=j}^i \tilde{\rho}_j^{k-1}
  \Delta\phi_k \tilde{\rho}_j^{m-1} \rho_m r_m
  \\
  &=\sum_{j=1}^i \sum_{k=1}^j \sum_{m=1}^k \tilde{\rho}_m^{j-1}
  \Delta\phi_j \tilde{\rho}_m^{k-1} \rho_k r_k
  +
  \sum_{j=2}^i \sum_{k=1}^{j-1} \sum_{m=1}^k \tilde{\rho}_m^{k-1}
  \Delta\phi_k  \tilde{\rho}_m^{j-1} \rho_j r_j.
\end{align}
Writing
\begin{equation}
  y_k = \sum_{m=1}^k (\tilde{\rho}_m^{k-1})^2 = 1 +
  (\gamma\lambda\rho_{k-1})^2 y_{k-1},
\end{equation}
we have that:
\begin{equation}
  \sum_{m=1}^k \tilde{\rho}_m^{j-1} \tilde{\rho}_m^{k-1} =
  \tilde{\rho}_k^{j-1} y_k.
\end{equation}
Therefore:
\begin{equation}
  \sum_{j=1}^i  \psi_{j\rightarrow i} z_{j\rightarrow i}
  = \sum_{j=1}^i \sum_{k=1}^j \tilde{\rho}_k^{j-1} y_k
  \Delta\phi_j\rho_k r_k
  +
  \sum_{j=2}^i \sum_{k=1}^{j-1} \tilde{\rho}_k^{j-1} y_k \Delta\phi_k \rho_j r_j.
\end{equation}
With the following notations:
\begin{align}
  z_j &= \sum_{k=1}^j \tilde{\rho}_k^{j-1} y_k \rho_k r_k
  = \gamma\lambda \rho_{j-1} z_{j-1} + \rho_j r_j y_j
  \\
  \mbox{and~~~}\mathfrak{D}_j &= \sum_{k=1}^j  \tilde{\rho}_k^{j-1} y_k \Delta\phi_k
  = \gamma\lambda\rho_{j-1} \mathfrak{D}_{j-1} + y_j \Delta\phi_j,
\end{align}
and with the convention that $z_0=0$ and $\mathfrak{D}_0 = 0$, one
can write:
\begin{equation}
  \sum_{j=1}^i  \psi_{j\rightarrow i} z_{j\rightarrow i}
  = \sum_{j=1}^i (\Delta\phi_j \rho_j r_j y_j +
  \gamma\lambda\rho_{j-1}(\Delta\phi_j z_{j-1} + \rho_j r_j
  \mathfrak{D}_{j-1}))
\end{equation}
Similarly, on can show that:
\begin{equation}
  \sum_{j=1}^i  \psi_{j\rightarrow i} \psi_{j\rightarrow i}^T
  = \sum_{j=1}^i (\Delta\phi_j \Delta\phi_j^T y_j +
  \gamma\lambda\rho_{j-1}(\Delta\phi_j \mathfrak{D}_{j-1}^T + \mathfrak{D}_{j-1}\Delta\phi_j^T))
\end{equation}
Denoting
  \begin{align}
    u_j &= \sqrt{y_j}\Delta\phi_j,\\
    v_j& = \frac{\gamma\lambda\rho_{j-1}}{\sqrt{y_j}}\mathfrak{D}_{j-1},
  \end{align}
and $I_2$ the $2\times 2$ identity matrix, we have:
  \begin{align}
    \sum_{j=1}^i \psi_{j\rightarrow i} \psi_{j\rightarrow i}^T
    &= \sum_{j=1}^i ((u_j+v_j)(u_j+v_j)^T - v_j v_j^T)
    \\
    &= \sum_{j=1}^{i-1} \psi_{j\rightarrow i} \psi_{j\rightarrow i}^T +
    \underbrace{\begin{pmatrix}
      u_i+v_i & v_i
    \end{pmatrix}}_{=U_i} I_2
    \underbrace{\begin{pmatrix}
      (u_i + v_i)^T \\ -v_i^T
    \end{pmatrix}}_{=V_i}.
  \end{align}
  We can apply the Woodbury identity given in Lemma~\ref{lemma:wmi}:
  \begin{align}
    C_i &= \left(\sum_{j=1}^i \psi_{j\rightarrow i} \psi_{j\rightarrow i}^T\right)^{-1}
    = \left(\sum_{j=1}^{i-1} \psi_{j\rightarrow i} z_{j\rightarrow i} +
    U_i
    I_2 V_i\right)^{-1}
    \\
    &= C_{i-1} - C_{i-1} U_i \left(I_2 + V_i C_{i-1} U_i\right)^{-1}
    V_i C_{i-1}.
  \end{align}
  The other sum can also be reworked:
  \begin{align}
    \tilde b_i &= \sum_{j=1}^i \psi_{j\rightarrow i} z_{j\rightarrow i}
    = \sum_{j=1}^i \Delta\phi_j r_j y_j + \gamma\lambda
    \left(\mathfrak{D}_{j-1}r_j + \Delta\phi_j z_{j-1}\right)
    \\
    &= \tilde b_{i-1} + \Delta\phi_i r_i y_i + \gamma\lambda
    \left(\mathfrak{D}_{i-1}r_i + \Delta\phi_i z_{i-1}\right)
    =  \tilde b_{i-1} + U_i\underbrace{
    \begin{pmatrix}
      \sqrt{y_i} r_i + \frac{\gamma\lambda}{\sqrt{y_i}} z_{i-1} \\ -
      \frac{\gamma\lambda}{\sqrt{y_i}} z_{i-1}
    \end{pmatrix}}_{=W_i}.
  \end{align}
  Finally, the recursive BRM($\lambda$) estimate can be computed as follows:
  \begin{equation}
    \theta_i = C_i \tilde b_i
    = \theta_{i-1} +  C_{i-1} U_i \left(I_2 + V_i C_{i-1}
    U_i\right)^{-1}\left(W_i - V_i \theta_{i-1}\right).
  \end{equation}
  This gives BRM($\lambda$)  as provided in Algorithm~\ref{algo:brm}.

\section{Proof of Theorem~\ref{th2} (Convergence of BRM($\lambda$))}

\label{proofbrm}

The proof of Theorem~\ref{th2} follows the general idea of that of Proposition 4 of \citet{Bertsekas:2009}. It is done in 2 steps. First we argue that the limit of the sequence is linked to that of an alternative algorithm for which one cuts the traces at a certain depth $l$. Then, we show that for all depth $l$, this alternative algorithm converges almost surely, we explicitely compute its limit and make $l$ tend to infinity to obtain the limit of BRM($\lambda$).

We will only show that $\frac{1}i \tilde A_i$ tends to $\tilde A$. The argument is similar for $\frac 1 i b_i \rightarrow \tilde b$.
Consider the following  $l$-truncated version of the algorithm based on the following alternative traces (we here limit the ``memory'' of the traces to a size $l$):
\begin{align}
y_{k,l} &= \sum_{m=\max(1,k-l+1)}^k (\tilde \rho_m^{k-1})^2 \\
\mathfrak{D}_{j,l}&=  \sum_{k=\max(1,j-l+1)}^j \tilde\rho_k^{j-1}
y_{k,l} \Delta \phi_k
\end{align}
and update the following matrix:
\begin{equation}
\tilde A_{i,l} = \tilde A_{i-1,l} + \Delta \phi_i\Delta \phi_i^T
y_{i,l} + \tilde \rho_{i-1}(\Delta \phi_i
\mathfrak{D}_{i-1,l}^T+\mathfrak{D}_{i-1,l}\Delta \phi_i^T).
\end{equation}
The assumption in \eqrefb{condBRM} implies that $\tilde \rho_i^{j-1} \le \beta^{j-i}$, therefore it can be seen that for all $k$,
\begin{equation}
  |y_{k,l}-y_k| = \sum_{m=1}^{\max(0,k-l)} (\tilde \rho_m^{k-1})^2 \leq \sum_{m=1}^{\max(0,k-l)} \beta^{2(k-m)}    \le \frac{\beta^{2l}}{1-\beta^2}=\epsilon_1(l)
\end{equation}
where $\epsilon_1(l)$ tends to 0 when $l$ tends to infinity.
Similarly, using the fact that $y_k \leq \frac{1}{1-\beta^2}$ and writing $K=\max_{s,s'} \| \phi(s)-\gamma \phi(s') \|_\infty$, one has for all $j$,
\begin{align}
\| \mathfrak{D}_{j,l}-\mathfrak{D}_j \|_\infty & \le \sum_{k=1}^{\max(0,j-l)}\tilde\rho_k^{j-1} \|y_k\Delta \phi_k\|_\infty +  \sum_{k=\max(1,j-l+1)}^j \tilde\rho_k^{j-1} |y_{k,l}-y_k| \|\Delta \phi_k\|_\infty\\
& \le \sum_{k=1}^{\max(0,j-l)}\tilde\rho_k^{j-1} \frac{1}{1-\beta^2}K + \sum_{k=\max(1,j-l+1)}^j \tilde\rho_k^{j-1} \frac{\beta^{2l}}{1-\beta^2}K \\
& \le \frac{\beta^l}{1-\beta}\frac{1}{1-\beta^2}K + \frac{1}{1-\beta}\frac{\beta^{2l}}{1-\beta^2}K=\epsilon_2(l)
\end{align}
where $\epsilon_2(l)$ also tends to 0. Then, it can be seen that:
\begin{align}
\|\tilde A_{i,l} - \tilde A_{i} \|_\infty & = \left\| \tilde A_{i-1,l}-\tilde A_{i-1} + \Delta \phi_i\Delta \phi_i^T (y_{i,l}-y_i) \right. \\
& \hspace{1cm} + \left. \tilde \rho_{i-1}(\Delta \phi_i (\mathfrak{D}_{i-1,l}^T-\mathfrak{D}_{i-1}^T)+(\mathfrak{D}_{i-1,l}-\mathfrak{D}_{i-1})\Delta \phi_i^T)\right\|_\infty\\
& \leq  \|\tilde A_{i-1,l} - \tilde A_{i-1}\|_\infty + \|\Delta \phi_i\Delta \phi_i^T\|_\infty|y_{k,l}-y_k| + 2 \beta \|\Delta \phi_i\|_\infty \| \mathfrak{D}_{i-1,l}-\mathfrak{D}_i \|_\infty \\
&\le \|\tilde A_{i-1,l} - \tilde A_{i-1}\|_\infty + K^2 \epsilon_1(l) + 2\beta K \epsilon_2(l)
\end{align}
and, by a recurrence on $i$, one obtains
\begin{equation}
\left \|\frac{\tilde A_{i,l}}{i} - \frac{\tilde A_{i}}{i} \right\|_\infty \leq \epsilon(l)
\end{equation}
where $\epsilon(l)$ tends to 0 when $l$ tends to infinity.
This implies that:
\begin{equation}
\liminf_{l \rightarrow \infty} \frac{\tilde A_{i,l}}{i}-\epsilon(l) \leq \liminf_{l \rightarrow \infty} \frac{\tilde A_{i}}{i} \leq \limsup_{l \rightarrow \infty} \frac{\tilde A_{i}}{i} \le \limsup_{l \rightarrow \infty} \frac{\tilde A_{i,l}}{i} + \epsilon(l).
\end{equation}
In other words, one can see that $ \lim_{i \rightarrow \infty} \frac{\tilde A_{i}}{i}$ and $\lim_{l \rightarrow \infty} \lim_{i \rightarrow \infty}\frac{\tilde A_{i,l}}{i}$ are equal if the latter exists. In the remaing of the proof, we show that the latter limit indeed exists and we compute it explicitely.

Let us fix some $l$ and let us consider the sequence $(\frac{\tilde
A_{i,l}}{i})$. At some index $i$, $y_{i,l}$ depends only on the last
$l$ samples, while $\mathfrak{D}_{i,l}$ depends on the same samples
and the last $l$ values of $y_{j,l}$, thus on the last $2l$ samples.
It is then natural to view the computation of  ${\tilde A_{i,l}}$,
which is based on $y_{i,l}$, $\mathfrak{D}_{i-1,l}$ and $\Delta
\phi_i=\phi_i-\gamma\rho_i \phi_{i+1}$, as being related to a Markov
chain of which the states are the $2l+1$ consecutive states of the
original chain $(s_{i-2l},\dots,s_i,s_{i+1})$. Write $E_0$ the
expectation with respect to its stationary distribution. By the
Markov chain Ergodic Theorem, we have with probability 1:
\begin{equation}
\label{eq:thelimit} \lim_{i \rightarrow \infty}\frac{\tilde
A_{i,l}}{i}=E_0\left[ \Delta \phi_{2l}\Delta \phi_{2l}^T y_{2l,l} +
\lambda\gamma \rho_{2l-1}(\Delta \phi_{2l}
\mathfrak{D}_{2l-1,l}^T+\mathfrak{D}_{2l-1,l}\Delta
\phi_{2l}^T)\right].
\end{equation}
Let us now explicitely compute this expectation.
Write $x_i$ the indicator vector (of which the $k^{th}$ coordinate
equals $1$ when the state at time $i$ is $k$ and $0$ otherwise). One
has the following relations: $\phi_i=\Phi^T x_i$.
Let us first look at the left part of the above limit:
{\scriptsize \begin{align}
E_0\left[\Delta \phi_{2l}\Delta \phi_{2l}^T y_{2l,l}\right]&=E_0\left[(\phi_{2l}-\gamma\rho_{2l}\phi_{2l+1})(\phi_{2l}-\gamma\rho_{2l}\phi_{2l+1})^T y_{2l,l}\right]\\
&=E_0\left[\Phi^T (x_{2l}-\gamma \rho_{2l}x_{2l+1})(x_{2l}-\gamma \rho_{2l}x_{2l+1})^T \Phi \left(\sum_{m=l+1}^{2l}(\lambda\gamma)^{2(2l-m)}(\rho_m^{2l-1})^2 \right)\right]\\
&=\Phi^T \left\{ \sum_{m=l+1}^{2l} (\lambda\gamma)^{2(2l-m)} E_0\left[ (\rho_m^{2l-1})^2 (x_{2l}-\gamma \rho_{2l}x_{2l+1})(x_{2l}-\gamma \rho_{2l}x_{2l+1})^T \right] \right\} \Phi\\
&=\Phi^T \left\{ \sum_{m=l+1}^{2l}  (\lambda\gamma)^{2(2l-m)}  E_0\left[ (X_{m,2l,2l}-\gamma X_{m,2l,2l+1}-\gamma X_{m,2l+1,2l}+\gamma^2 X_{m,2l+1,2l+1}) \right] \right\} \Phi
\end{align}}
where we used  the definiton $\tilde \rho_j^{k-1}=(\lambda\gamma)^{k-j}\rho_j^{k-1}$ and the notation $X_{m,i,j}=\rho_m^{i-1}\rho_m^{j-1}x_i x_j^T$.
To finish the computation, we will mainly rely on the following Lemma:
\begin{lemma}[Some identities]\label{lemma:ids}
\label{identities}
Let $\tilde P$ be the matrix of which the coordinates are
$\tilde p_{s s'}=\sum_a \pi(s,a)\rho(s,a)T(s,a,s')$, which is in
general not a stochastic matrix. Let $\mu_0$ be the stationary distribution of the behavior policy $\pi_0$. Write $\tilde D_i=\diag\left((\tilde P^T)^{i} \mu_0 \right)$. Then
\begin{align}
\forall m \le i, ~ E_0[X_{m,i,i}] &= \tilde D_{i-m}   \\
\forall m \le i \le j, ~ E_0[X_{m,i,j}]&=\tilde D_{i-m} P^{j-i}\\
\forall m \le j \le i, ~ E_0[X_{m,i,j}]&=(P^T)^{j-i}\tilde D_{i-m}
\end{align}
\end{lemma}
\begin{proof}
We first observe that:
\begin{eqnarray*}
E_0[X_{m,i,i}] &=& E_0[(\rho_m^{i-1})^2 x_i x_i^T] \\
& = & E_0[(\rho_m^{i-1})^2 \diag(x_i)] \\
& = & \diag\left( E_0[(\rho_m^{i-1})^2 x_i\right)
\end{eqnarray*}
To provide the identity, we will thus simply provide a proof by
recurrence that $E_0[(\rho_m^{i-1})^2 x_i]=(\tilde P^T)^{m-i}\mu_0$. For
$i=m$, we have $E_0[x_m]=\mu_0$. Now suppose the relation holds for $i$
and let us prove it for $i+1$.
\begin{eqnarray*}
E_0[(\rho_m^i)^2 x_{i+1}] & = & E_0 \left[E_0[(\rho_m^i)^2 x_{i+1}|{\cal F}_i] \right] \\
& = & E_0\left[E_0[(\rho_m^{i-1})^2 (\rho_i)^2 x_{i+1}|{\cal F}_i] \right] \\
& = & E_0\left[(\rho_m^{i-1})^2 E_0[(\rho_i)^2 x_{i+1}|{\cal F}_i]\right].
\end{eqnarray*}
Write ${\cal F}_i$ the realization of the process until time $i$.
Recalling that $s_i$ is the state at time $i$ and $x_i$ is the
indicator vector corresponding to $s_i$, one has for all $s'$:
\begin{eqnarray*}
E_0[(\rho_i)^2 x_{i+1}(s')|{\cal F}_i] & = & \sum_{a}\pi_0(s_i,a)\rho(s_i,a)^2 T(s_i,a,s') \\
& = & \sum_{a}\pi(s_i,a)\rho(s_i,a)T(s_i,a,s') \\
& = &  \tilde p_{s_i,s'} \\
& = & [\tilde P^T x_i](s').
\end{eqnarray*}
As this is true for all $s'$, we deduce that $E_0[(\rho_i)^2 x_{i+1}|{\cal F}_i]=\tilde P^T x_i$ and
\begin{eqnarray*}
E_0[(\rho_m^i)^2 x_{i+1}] &  = & E_0[(\rho_m^{i-1})^2 \tilde P^T x_i] \\
& = & \tilde P^T E_0[(\rho_m^{i-1})^2 \tilde P^T x_i] \\
& = & \tilde P^T (\tilde P^T)^{i} \mu_0 \\
& = & (\tilde P^T)^{i+1}\mu_0
\end{eqnarray*}
which concludes the proof by recurrence.

Let us consider the next identity. For $i \leq j$,
\begin{eqnarray*}
E_0[\rho_m^{i-1} \rho_m^{j-1}x_i x_j^T]& = & E_0[E_0[\rho_m^{i-1} \rho_m^{j-1} x_i x_j^T | {\cal F}_i]] \\
& = & E_0[(\rho_m^{i-1})^2 x_i E_0[ \rho_i^{j-1}x_j^T | {\cal F}_i]] \\
& = & E_0[(\rho_m^{i-1})^2 x_i x_i^T P^{j-i}] \\
& = & \diag\left( (\tilde P^T)^{m-i}\mu_0 \right) P^{j-i}.\\
\end{eqnarray*}

Eventually, the last identity is obtained by considering $Y_{m,i,j}=X_{m,j,i}^T. $
\end{proof}

Thus, coming back to our calculus,
{\small \begin{align}
E_0\left[\Delta \phi_{2l}\Delta \phi_{2l}^T y_{2l,l}\right]&=\Phi^T \left\{ \sum_{m=l+1}^{2l}  (\lambda\gamma)^{2(2l-m)} \left(\tilde D_{2l-m}-\gamma \tilde D_{2l-m} P - \gamma P^T \tilde D_{2l-m} + \gamma^2 \tilde D_{2l+1-m} \right) \right\}\Phi\\
&=\Phi^T (D_l - \gamma D_l P - \gamma P^T D_l + \gamma^2 D'_l )\Phi \label{eq:part1}
\end{align}}
\begin{equation}
\mbox{with~~~}D_l  =  \sum_{j=0}^{l-1} (\lambda\gamma)^{2j} \tilde D_j,~~~\mbox{and}~~~
 D_l'  =  \sum_{j=0}^{l-1} (\lambda\gamma)^{2j} \tilde D_{j+1}.
\end{equation}

Similarly, the second term on the right side of \eqrefb{thelimit} satisfies:
{\scriptsize
\begin{align}
&E_0\left[ \rho_{2l-1}\mathfrak{D}_{2l-1,l}\Delta \phi_{2l}^T\right]=E_0\left[ \rho_{2l-1}\sum_{k=l}^{2l-1}\tilde \rho_k^{2l-2}y_{k,l}\Delta \phi_k \Delta \phi_{2l}^T \right]\\
&=E_0\left[\sum_{k=l}^{2l-1}(\lambda\gamma)^{2l-1-k}\rho_k^{2l-1}\left(\sum_{m=k-l+1}^k(\tilde \rho_m^{k-1})^2\right)\Phi^T (x_k-\gamma \rho_k x_{k+1})(x_{2l}-\gamma \rho_{2l}x_{2l+1})^T \Phi \Delta \phi_{2l}^T \right]\\
&=\Phi^T \left( \sum_{k=l}^{2l-1} (\lambda\gamma)^{2l-1-k}\sum_{m=k-l+1}^k (\lambda\gamma)^{2(k-m)} E_0\left[ \rho_m^{2l-1} \rho_m^{k-1}(x_k-\gamma \rho_k x_{k+1})(x_{2l}-\gamma \rho_{2l}x_{2l+1})^T\right] \right) \Phi\\
&= \Phi^T \left(  \sum_{k=l}^{2l-1} (\lambda\gamma)^{2l-1-k}\sum_{m=k-l+1}^k (\lambda\gamma)^{2(k-m)}  E_0\left[X_{m,k,2l}-\gamma X_{m,k+1,2l}-\gamma X_{m,k,2l+1}+\gamma^2 X_{m,k+1,2l+1}\right] \right) \Phi \\
&= \Phi^T \left(  \sum_{k=l}^{2l-1} (\lambda\gamma)^{2l-1-k}\sum_{m=k-l+1}^k (\lambda\gamma)^{2(k-m)}  \left( \tilde D_{k-m}P^{2l-k} -\gamma \tilde D_{k+1-m}P^{2l-k-1} -\gamma \tilde D_{k-m}P^{2l+1-k} + \gamma^2 \tilde D_{k+1-m}P^{2l-k}  \right) \right) \Phi \\
&= \Phi^T \left(  \sum_{k=l}^{2l-1} (\lambda\gamma)^{2l-1-k}\sum_{m=k-l+1}^k (\lambda\gamma)^{2(k-m)}  \left( \tilde D_{k-m}P^{2l-k}(I-\gamma P) - \gamma \tilde D_{k+1-m}P^{2l-1-k}(I-\gamma P) \right) \right) \Phi \\
&= \Phi^T \left(  \sum_{k=l}^{2l-1} (\lambda\gamma)^{2l-1-k}\sum_{m=k-l+1}^k (\lambda\gamma)^{2(k-m)}  \left( \tilde D_{k-m}P - \gamma \tilde D_{k+1-m}\right) P^{2l-1-k} (I-\gamma P) \right) \Phi \\
&=\Phi^T \left(  \sum_{k=l}^{2l-1} (\lambda\gamma)^{2l-1-k}  \left( D_l P - \gamma D'_l\right) P^{2l-1-k} (I-\gamma P) \right) \Phi \\
&=\Phi^T \left( D_l P - \gamma D'_l\right) Q_l (I-\gamma P)  \Phi
\end{align}}
with $Q_l=\sum_{j=0}^{l-1} (\lambda\gamma P)^j.$

Gathering this and \eqrefb{part1}, we see that the limit of $\frac{A_{i,l}}{i}$ expressed in \eqrefb{thelimit} equals:
\begin{equation}
\Phi^T\left[ D_l - \gamma D_l P - \gamma P^T D_l + \gamma^2 D'_l + \lambda\gamma \left( (D_l P - \gamma D'_l)Q_l (I-\gamma P) + (I-\gamma P^T)Q_l^T(P^T D_l-\gamma D'_l) \right) \right]\Phi.
\end{equation}
When $l$ tends to infinity, $Q_l$ tends to $Q=(I-\lambda\gamma P)^{-1}$.
The assumption of \eqrefb{condBRM} ensures that $(\lambda\gamma)\tilde P$ has spectral radius smaller than 1, and thus when $l$ tends to infinity, $D_l$ tends to $D=\diag\left((I-(\lambda\gamma)^2 \tilde P^T)^{-1}\mu_0\right)$ and $D_l'$ to $D'=\diag\left(\tilde P^T(I-(\lambda\gamma)^2 \tilde P^T)^{-1}\mu_0\right)$. In other words, $\lim_{l \rightarrow \infty}\lim_{i \rightarrow \infty} \frac{\tilde A_{i,l}}{i}$ exists with probability 1 and equals:
\begin{equation}
\Phi^T\left[ D - \gamma D P - \gamma P^T D + \gamma^2 D' + \lambda\gamma \left( (D P - \gamma D')Q (I-\gamma P) + (I-\gamma P^T)Q^T(P^T D-\gamma D') \right) \right]\Phi.
\end{equation}
Eventually, this shows that $\lim_{i \rightarrow \infty} \frac{\tilde A_{i}}{i}$ exists with probability 1 and shares the same value.

A similar reasoning allows to show that $\lim_{i \rightarrow \infty} \frac{\tilde b_{i}}{i}$ exists and equals
\begin{equation}
\Phi^T \left[ (I-\gamma P^T)Q^T D + \lambda\gamma (D P-\gamma D')Q \right] R^\pi. ~~~\qed
\end{equation}

\section{Proof of Proposition~\ref{prop:beurk}}

\label{appbeurk}

To prove Proposition~\ref{prop:beurk}, we need the following technical
lemma.
\begin{lemma}
\label{lemma:technical}
  Forget the notations used so far.
  Let $\alpha_i$ and $\beta_i$ be two forward recursions defined as
  \begin{align}
    \alpha_i &= a_i + \eta_i \alpha_{i+1}
    \\
    \text{ and }
    \beta_i &= b_i + \eta_i \beta_{i+1}.
  \end{align}
  Assume that for any function $f$ we have that\footnote{This is
  typically true if the index $i$ refers to a state sampled
  according to some stationary distribution, which is the case we
  are interested in.}
  \begin{equation}
    E[f(a_i,b_i,\eta_i)] = E[f(a_{i-1},b_{i-1},\eta_{i-1})].
  \end{equation}
  Let also $u_i$, $v_i$ and $w_i$ be the backward recursions defined
  as:
  \begin{align}
    w_i &= 1 + \eta_{i-1}^2 w_{i-1}
    \\
    u_i &= a_i w_i + \eta_{i-1} u_{i-1}
    \\
    v_i &= b_i w_i + \eta_{i-1} v_{i-1}
  \end{align}
  Then, we have:
  \begin{equation}
    E[\alpha_i\beta_i] = E[a_i v_i + b_i u_i - a_i b_i w_i]
  \end{equation}
\end{lemma}
\begin{proof}
  The proof looks like the one of Proposition~\ref{prop:phi_delta}, but is a little bit more
  complicated. A key equality, to be applied repeatedly, is:
  \begin{align}
    \alpha_i \beta_i &= (a_i + \eta_i \alpha_{i+1})(b_i + \eta_i
    \beta_{i+1})
    \\
    &= a_i \beta_i + b_i \alpha_i + \eta_i^2 \alpha_{i+1}
    \beta_{i+1} - a_i b_i.
  \end{align}
  Another equality to be used repeatedly makes use of the
  ``stationarity''
  assumption. For any $k\geq 0$ we have:
  \begin{equation}
    E[(\prod_{j=0}^k \eta_{i-j}^2) \alpha_{i+1}\beta_{i+1}] =
    E[(\prod_{j=1}^{k+1} \eta_{i-j}^2) \alpha_{i}\beta_{i}].
  \end{equation}
  These two identities can be used to work the term of interest:
  \begin{align}
    E[\alpha_i\beta_i] &= E[(a_i + \eta_i \alpha_{i+1})(b_i + \eta_i \beta_{i+1})]
    \\
    &= E[a_i \beta_i] + E[b_i\alpha_i] + E[\eta_i^2
    \alpha_{i+1}\beta_{i+1}] - E[a_i b_i]
    \\
    &= E[a_i \beta_i] + E[b_i\alpha_i] - E[a_i b_i] + E[\eta_{i-1}^2(a_i + \eta_i \alpha_{i+1})(b_i + \eta_i \beta_{i+1})]
    \\
    &= E[a_i (1 + \eta_{i-1}^2) \beta_i] + E[b_i (1 +
    \eta_{i-1}^2)\alpha_i] - E[a_i b_i (1 +
    \eta_{i-1}^2)] + E[(\eta_{i-1}\eta_i)^2
    \alpha_{i+1}\beta_{i+1}]
  \end{align}
  This process can be repeated, giving
  \begin{equation}
    E[\alpha_i\beta_i] = E[(a_i\beta_i + b_i \alpha_i - a_i b_i) (1 + \eta_{i-1}^2 + (\eta_{i-1}\eta_{i-2})^2 + \dots)].
  \end{equation}
  We have that
  \begin{equation}
    w_i = 1 + \eta_{i-1}^2 w_{i-1} = 1 + \eta_{i-1}^2 + (\eta_{i-1}\eta_{i-2})^2
    + \dots,
  \end{equation}
  therefore:
  \begin{equation}
    E[\alpha_i\beta_i] = E[a_i w_i \beta_i] + E[b_i w_i \alpha_i] -
    E[a_i b_i w_i]
  \end{equation}
  We can work on the first term:
  \begin{align}
    E[a_i w_i \beta_i] &= E[a_i w_i (b_i + \eta_i \beta_{i+1})]
    \\
    &= E[a_i w_i b_i] + E[a_{i-1} w_{i-1} \eta_{i-1}(b_i + \eta_i \beta_{i+1})]
    \\
    &= E[b_i(a_i w_i + \eta_{i-1} (a_{i-1}w_{i-1}) +
    \eta_{i-1}\eta_{i-2} (a_{i-2}w_{i-2}) + \dots)]
    \\
    &= E[b_i u_i].
  \end{align}
  The work on the second term is symmetric:
  \begin{equation}
    E[b_i w_i \alpha_i] = E[a_i v_i].
  \end{equation}
  This finishes proving the result.
\end{proof}

The proof of Proposition~\ref{prop:beurk} is a simple application of the
preceding technical lemma.   By lemma~\ref{lemma:td_fwd_rec}, we
have that
  \begin{equation}
    \underbrace{\delta_i^\lambda}_{\doteq \alpha_i} = \underbrace{\delta_i}_{\doteq a_i} + \underbrace{\gamma \lambda
    \rho_i}_{\doteq \eta_i}
    \underbrace{\delta_{i+1}^\lambda}_{\doteq \alpha_{i+1}}.
  \end{equation}
  By lemma~\ref{lemma:gradTrec}, we have that
  \begin{equation}
    \underbrace{g_i^\lambda}_{\doteq \beta_i} = \underbrace{\gamma \rho_i (1-\lambda) \phi_{i+1}}_{\doteq b_i} + \underbrace{\gamma \lambda
    \rho_i}_{\doteq \eta_i} \underbrace{g_{i+1}^\lambda}_{\doteq
    \beta_{i+1}}.
  \end{equation}
  The result is then a direct application of
  lemma~\ref{lemma:technical}.


\end{document}